\documentclass{article}

\usepackage{tcolorbox}
\usepackage[dvipsnames]{xcolor} 
\usepackage{tcolorbox}
\usepackage{algorithm}
\usepackage{algorithmicx}
\usepackage{algpseudocode}

\usepackage[main, final]{neurips_2026}

\usepackage[utf8]{inputenc} 
\usepackage[T1]{fontenc}    
\usepackage{hyperref}       
\usepackage{url}            
\usepackage{booktabs}       
\usepackage{amsfonts}       
\usepackage{nicefrac}       
\usepackage{microtype}      
\usepackage{xcolor}         

\usepackage{booktabs}
\usepackage{adjustbox}
\usepackage{graphicx}
\usepackage{amsmath,amsthm,amssymb}
\usepackage{thmtools}
\usepackage{multirow}
\usepackage{bbm}
\usepackage{authblk}
\usepackage{kotex}
\usepackage{multicol}
\usepackage{wrapfig,lipsum}
\usepackage{subcaption}
\usepackage{makecell}
\usepackage{multirow}
\newtheorem{theorem}{Theorem}[section]
\newtheorem{proposition}[theorem]{Proposition}

\usepackage{amsfonts}
\usepackage{enumerate}
\usepackage{algorithmicx}
\usepackage{algpseudocode}
\usepackage{algcompatible}
\usepackage{setspace}
\usepackage{enumitem}
\usepackage{multirow}
\usepackage{graphicx}
\usepackage{caption,stackengine}
\usepackage{bm}

\title{A Composite Activation Function\\
  for Learning Stable Binary Representations}

\author{%
  Seokhun Park$^{1}$, Choeun Kim$^{2}$, Kwanho Lee$^{2}$, Sehyun Park$^{2}$, Insung Kong$^{3}$,  Yongdai Kim$^{2}$ \\
  $^{1}$Meritz Fire \& Marine Insurance\\
  $^{2}$Department of Statistics, Seoul National University\\
  $^{3}$Department of Mathematical Sciences, UNIST\\ 
  \texttt{seokhunpark1366@gmail.com},\\
  \texttt{\{kimchoeun, khlee0527, ps\_hyen\}@snu.ac.kr},\\
  \texttt{kong@unist.ac.kr}, \\
  \texttt{ydkim0903@gmail.com}\\
}

\begin{document}

\maketitle

\begin{abstract}
Activation functions play a central role in neural networks by shaping internal representations. 
Recently, learning binary activation representations has attracted significant attention due to their advantages in computational and memory efficiency, as well as interpretability. 
However, training neural networks with Heaviside activations remains challenging, as their non-differentiability obstructs standard gradient-based optimization.
In this paper, we propose \textit{Heavy-Tailed Activation Function (HTAF)}, a smooth approximation to the Heaviside function that enables stable training with gradient-based optimization. 
We construct HTAF as a sigmoid–hyperbolic tangent composite function and theoretically show that it maintains a large gradient mass around zero inputs while exhibiting slower gradient decay in the tail regions.
We show that Spiking Neural Networks, Binary Neural Networks and Deep Heaviside neural Networks can be trained stably using HTAF with gradient-based optimization.
Finally, we introduce Implicit Concept Bottleneck Models (ICBMs), an interpretable image model that leverages HTAF to induce discrete feature representations.
Extensive experiments across various architectures and image datasets demonstrate that ICBM enables stable discretization while achieving prediction performance comparable to or better than standard models. 
\end{abstract}

\section{Introduction}
Activation functions play a central role in neural networks by transforming continuous signals into structured representations that enable effective learning and prediction. 
Through this transformation, activation functions critically influence not only prediction performance (\citep{bingham2023efficient, hayou2019impact}), but also the robustness (\citep{singla2021low, xie2020smooth}) and interpretability of learned representations (\citep{chen2024neural, fel2023craft, kim2018interpretability, ghorbani2019towards}).
Early neural network models (\cite{rosenblatt1958perceptron}) were based on threshold-based activation functions, which are equivalent to the Heaviside activation function $x \mapsto \mathbb{I}(x \geq 0)$. 
However, such non-differentiable activations made gradient-based learning intractable. 
To address this issue, differentiable surrogate activation functions such as the sigmoid function were introduced (\cite{rumelhart1986learning}). 
Subsequently, to mitigate the gradient vanishing problem, Rectified Linear Units (ReLU) became widely adopted (\cite{glorot2011deep}).

Recently, the Heaviside activation function has attracted increasing attention again in the AI community due to their potential benefits in memory and computational efficiency (\citep{yin2019understanding, ergen2023globally, qin2020binary}), interpretability (\citep{wu2015adjustable, yue2024learning, leblanc2024seeking}), and connections to biological neural networks (\citep{zenke2018superspike, li2021differentiable, tavanaei2019deep}). 
Furthermore, the Heaviside activation function is closely related to quantized and binarized neural networks (\citep{courbariaux2016binarized, hubara2018quantized, qin2020binary}), which have become increasingly important for energy-efficient deployment of 
large deep neural networks such as large language models (LLMs, \citep{wang2023bitnet, ma2024era}).

Although the Heaviside activation function offers several advantages, training neural networks that use it with gradient-based optimization remains difficult because the function is inherently non-differentiable.
To address this issue, a substantial body of work (\cite{bengio2013estimating, li2021differentiable, zenke2018superspike, iliev2017approximation, neftci2019surrogate}) has focused on developing methods for learning neural networks with the Heaviside activation function.
\cite{bengio2013estimating} proposed Straight-Through Estimator (STE), which enables gradient-based optimization of non-differentiable functions by using surrogate gradients.
While STE serves as a practical workaround, it introduces a mismatch between the forward and backward passes and often leads to biased or suboptimal optimization (\cite{yin2019understanding}).
Another approach is to replace the Heaviside activation function with a smooth surrogate (\cite{iliev2017approximation}), most notably the sigmoid function, and encourage the activations to approach binary values by increasing the slope or scale of the nonlinearity. 
While this approach enables the use of standard gradient-based optimization, it suffers from severe optimization instability (\citep{glorot2010understanding, chen2024training}) including the gradient vanishing problem.

In this paper, we propose a novel smooth activation function, called \textit{Heavy-Tailed Activation Function (HTAF)}, which closely approximates the Heaviside function while enabling stable and accurate training with standard gradient-based optimization.
The key idea is to design an activation function that preserves a large gradient near zero (to approximate the Heaviside function closely), while allowing the gradient to decay slowly for large-magnitude inputs
(to avoid the gradient vanishing problem).
Specifically, HTAF has two parameters that allow the gradient magnitude in the tail regions to be adjusted without affecting the gradient near zero.
That is, by adjusting these two parameters, HTAF allows the gradient in the tail region to decay more slowly, while still preserving a high gradient near zero.

As an application of HTAF, we consider two problems.
Firstly, using HTAF, we propose a general learning framework for neural networks with the Heaviside activation function, including Spiking Neural Networks (SNNs, \citep{tavanaei2019deep, fang2021deep, maass1997networks}), Binary Neural Networks (BNNs, \cite{qin2020binary,rastegari2016xnor,liu2018bi}), and Deep Heaviside Neural Networks (DHNs, \cite{kong2025expressivity,ergen2023globally}), enabling stable end-to-end optimization without surrogate gradients for activation functions.
Second, we propose Implicit Concept Bottleneck Models (ICBMs), an extension of Concept Bottleneck Models (CBMs, \cite{koh2020concept}), which employ HTAF as the activation function in the final layer of the backbone feature extractor during training.
HTAF encourages each feature to take values close to either zero or one.
At inference time, HTAF is replaced by the Heaviside function, yielding exact binary activations and resulting in fully discrete and interpretable features without substantially sacrificing accuracy.

To sum up, our key contributions are summarized as follows.

\begin{enumerate}
    \item We introduce \textit{Heavy-Tailed Activation Function (HTAF)}, a novel smooth activation function, which approximates the Heaviside activation function.
    
    \item We theoretically show that, by appropriately tuning its two parameters, 
     HTAF can control the rate of gradient decay in tail regions while keeping a large gradient near zero, thereby enabling stable learning of deep neural networks (DNNs) with HTAF activation with standard gradient-based optimization.
    
    \item We demonstrate that HTAF enables gradient-based learning of DNNs with the Heaviside activation function, allowing stable end-to-end optimization and improving prediction performance over existing training methods, including those based on surrogate gradients.

    \item We propose Implicit Concept Bottleneck Models (ICBM), an interpretable image model that leverages HTAF to induce binary latent representations corresponding to implicit concepts. We show that the resulting model is interpretable while achieving competitive or improved prediction performance across various image datasets compared to non-interpretable models.

\end{enumerate}

\section{Related works}

\subsection{Deep Neural Networks with the Heaviside activation function}
\label{sec:learning_binary_act}

\paragraph{Heaviside-Activated Neural Networks and their advantages.}
DNNs with the Heaviside activation function have been widely studied in various neural network models, including Spiking Neural Networks (SNNs, \cite{zenke2018superspike,fang2021deep,hundrieser2025universal}), Binary Neural Networks (BNNs, \cite{qin2020binary,rastegari2016xnor}), and Deep Heaviside Neural Networks (DHNs, \cite{kong2025expressivity}).
SNNs use the Heaviside function to model discrete spike generation, offering a biologically inspired and energy-efficient computation paradigm. 
Similarly, in BNNs, activations and weights are constrained to binary values, allowing significant reductions in memory usage and computational cost.  
DHNs are standard deep neural networks employing the Heaviside activation function.

A key advantage of such models is their efficiency during inference.  
The output of the Heaviside activation function can be represented using a single bit, leading to substantial memory savings compared to floating-point representations.
Moreover, multiplications can be simplified or eliminated, resulting in lightweight computations and faster inference.

\paragraph{Training methods.}
Methods for training neural networks with the Heaviside activation function can be broadly categorized into three approaches.
The first approach (\cite{bengio2013estimating, li2021differentiable, zenke2018superspike}) employs the Heaviside activation function together with surrogate gradients during training.
A representative approach is Straight-Through Estimator (STE, \cite{bengio2013estimating}), which replaces the gradient of 
the Heaviside function with a smooth surrogate function during backpropagation while retaining the original hard thresholding in the forward pass.
Although surrogate-gradient methods have been widely studied and adopted, learning neural networks with the Heaviside activation function via such approaches remains challenging.
Specifically, the reliance on approximate rather than true gradients often leads to unstable optimization dynamics and degraded prediction performance (\cite{yin2019understanding}).

The second line of work (\cite{iliev2017approximation, cao2017hashnet, lahoud2019self}) approximates the Heaviside activation function with smooth activations to enable standard gradient-based optimization, most commonly using the sigmoid function. However, sigmoid-based approximations suffer from the well-known gradient vanishing problem (\citep{glorot2010understanding, chen2024training}).
Specifically, as the sigmoid becomes steeper around $x=0$, inputs quickly enter saturation regimes where gradients vanish, resulting in unstable training or optimization stagnation.

Finally, another line of work (\cite{diehl2015fast, rueckauer2017conversion, huang2024billm, sengupta2019going, cao2015spiking}) first trains conventional DNNs with a smooth activation function and then converts them post-hoc into SNNs or BNNs, thereby avoiding direct training with Heaviside activation function.

\subsection{Concept Bottleneck Models}

\cite{koh2020concept} introduced Concept Bottleneck Models (CBMs), a class of interpretable image models that incorporate human-interpretable concepts as intermediate representations.
In CBMs, the model first predicts a set of predefined concepts and then uses them for class prediction, enabling transparency and intervention at the concept level.

However, CBMs often suffer from reduced prediction performance and limited practicality due to the requirement of concept annotations, which are costly to obtain.
To address this, subsequent works have proposed label-free variants using pretrained models such as Contrastive Language-Image Pre-training (CLIP, \cite{radford2021learning}) and Large Language Models (LLMs, \cite{brown2020language}) to automatically generate concept labels for each image (\cite{oikarinen2023label, yuksekgonul2022post, kulkarni2025interpretable}).

Despite these advances, such approaches still achieve lower prediction performance than standard image models.
Moreover, concept annotations produced by multimodal models such as CLIP can be unreliable and may fail to accurately capture underlying semantic concepts of individual images (\citep{yuksekgonul2022and, kazmierczak2025enhancing}).

\section{Heavy-Tailed Activation Function and its application}
\subsection{Notation}
For a given positive integer $p$, we denote $\{1,...,p\}$ as $[p]$.
Let $\mathbb{R}_{+}$ denote the set of positive real numbers.
For a given input $x \in \mathbb{R}$, we let $\sigma(\cdot)$ and $\psi(\cdot)$ as the sigmoid and hyperbolic tangent functions, respectively, i.e.,
\begin{align*}
\sigma(x) := {1\over 1+\exp(-x)}\:\:\text{and}\:\: \psi(x) := {\exp(x)-\exp(-x)\over \exp(x)+\exp(-x)}.
\end{align*}

\subsection{Heavy-Tailed Activation Function}

For a given input $x \in \mathbb{R}$, we define {Heavy-Tailed Activation Function (HTAF)} as 
\begin{align}
\text{HTAF}(x|\alpha_{0},\alpha_{1}) := \sigma(\alpha_{1}\psi(\alpha_{0} x)),
\label{eq:HTAF}
\end{align}
where $\alpha_{0} \in \mathbb{R}_{+}$ and $\alpha_{1} \in \mathbb{R}_{+}$ are hyperparameters.

\begin{figure}[t]
    \centering
    \includegraphics[width=1.0\linewidth]{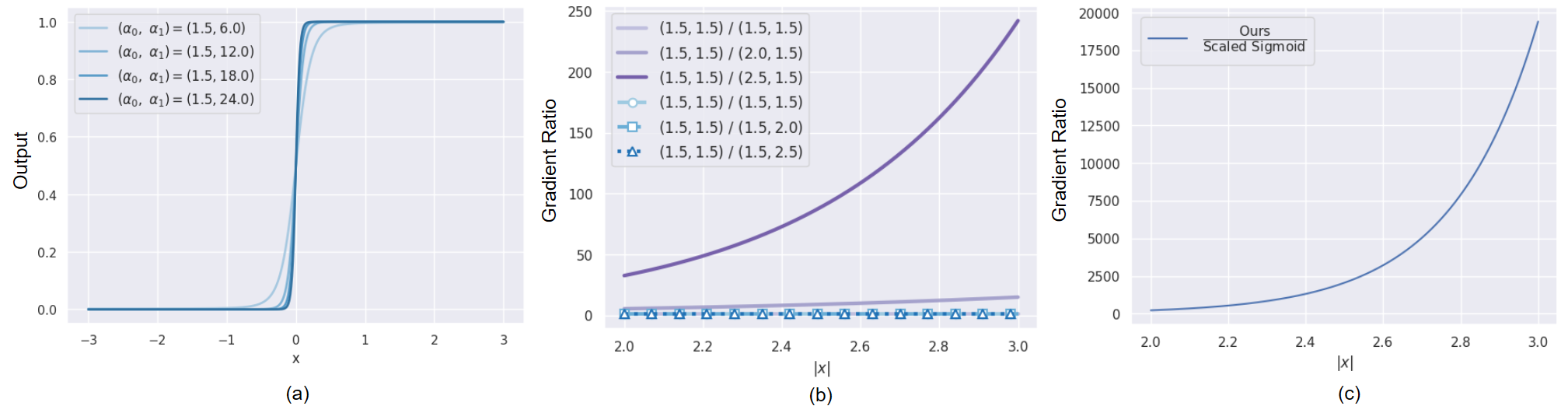}
    \caption{\textbf{Outputs and gradients of HTAF.}
(a) Output curves of HTAF with fixed $\alpha_{0}=1.5$ and varying $\alpha_{1} \in \{6.0, 12.0, 18.0, 24.0\}$.
(b) Ratios of gradients relative to the baseline case $\alpha_{0} = \alpha_{1} = 1.5$, obtained by varying $\alpha_{0}$ with $\alpha_{1}$ fixed at 1.5 and varying $\alpha_{1}$ with $\alpha_{0}$ fixed at 1.5. Blue curves correspond to varying $\alpha_{1}$, while purple curves correspond to varying $\alpha_{0}$.
(c) Ratio of derivatives between HTAF and Scaled Sigmoid in the tail region. 
}
\label{fig_all_plots}
\end{figure}

HTAF has two key properties. 
First, it induces binarized (near-binary) outputs, i.e., $\mathrm{HTAF}(x|\alpha_{0},\alpha_{1}) \approx 0\:\:\text{or}\:\: 1$
for all $x \in \mathbb{R}$ with an appropriate choice of $\alpha_{0}$ and $\alpha_{1}$, as shown in Figure \ref{fig_all_plots}. 
Second, its gradient decays slowly in the tail regions while maintaining a large gradient magnitude around $x=0$, ensuring stable gradient flow even for large input magnitudes while still closely approximating the Heaviside function.
Due to these properties, HTAF enables stable learning of binary representations using a standard gradient descent algorithm. 
In the following, we describe these two properties in detail.

\begin{proposition}
\label{thm:heavi_approx}
For any $\epsilon \in (0,1/2)$, suppose that $\alpha_0$ and $\alpha_1$ satisfy
\begin{align*}
\alpha_{1} > \log \left({1-\epsilon \over \epsilon}\right) \quad \text{and} \quad \alpha_{0} \geq 
{1\over 2\epsilon}
\log \left(
{\alpha_{1} + \log\left({1-\epsilon \over \epsilon}\right)
\over
\alpha_{1} - \log\left({1-\epsilon \over \epsilon}\right)}
\right).
\end{align*}
Then, we have
\begin{equation}
\label{eq:approx-HTAF}
\sup_{x \geq \epsilon}
\left|1-\text{HTAF}(x|\alpha_{0},\alpha_{1})\right| 
\leq \epsilon 
\quad \text{and} \quad 
\sup_{x \leq -\epsilon}
\left|\text{HTAF}(x|\alpha_{0},\alpha_{1})\right| 
\leq \epsilon.
\end{equation}
\end{proposition}

\begin{wrapfigure}{r}{0.45\textwidth}
    \centering
    \vspace{-10pt}
    \includegraphics[width=0.43\textwidth]{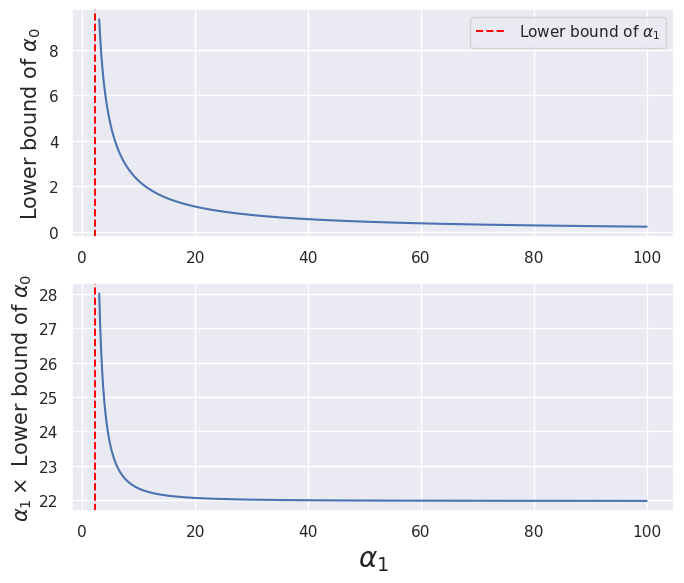}
    \caption{\textbf{Plots of the lower bound of $\alpha_{0}$ and $\alpha_{1}$ times the lower bound of $\alpha_{0}$ as functions of $\alpha_{1}$ for $\epsilon = 0.1$.}}
    \label{fig:ratio_lowerbound}
    \vspace{-10pt}
\end{wrapfigure}

Proposition \ref{thm:heavi_approx} presents the explicit sufficient conditions on the parameters $\alpha_0$ and $\alpha_1$ under which HTAF uniformly approximates the Heaviside function outside an arbitrarily small neighborhood of the origin. 
In particular, for any prescribed accuracy $\epsilon \in (0,1/2)$, choosing $\alpha_0$ and $\alpha_1$ according to the stated bounds guarantees that HTAF is within $\epsilon$ of the Heaviside function on $(-\infty,-\epsilon] \cup [\epsilon,\infty)$. 

Based on Proposition \ref{thm:heavi_approx}, we next discuss how to select the parameters $\alpha_{0}$ and $\alpha_{1}$ in practice.
Figure \ref{fig:ratio_lowerbound} illustrates the lower bounds of $\alpha_{0}$ and the product of $\alpha_{1}$ and the lower bound of $\alpha_{0}$ as functions of $\alpha_{1}$ for $\epsilon = 0.1$.
Furthermore, we observe that the product of $\alpha_{1}$ and the lower bound of $\alpha_{0}$ converges to 22 as $\alpha_{1}$ increases.
Therefore, one may consider choosing a small $\alpha_{0}$ and a large $\alpha_{1}$ while satisfying the condition $\alpha_{0}\alpha_{1} = 22$.
See Appendix \ref{app_selct_alpha01} for an ablation study on this choice.

An important message of Proposition \ref{thm:heavi_approx} is that a neural network with the Heaviside activation function can be approximated
by a DNN with HTAF activation with $O(\epsilon)$ approximation error where the constant term in $O(\epsilon)$
depends on the architecture of the neural network (i.e. the number of layers and number of nodes). See
Appendix \ref{app_dhn_approx} for a theoretical justification.
We now provide a theoretical justification for the phenomenon observed in Figure~\ref{fig:ratio_lowerbound}.
\begin{theorem}\label{thm:asymp_lower_bound}
For $\varepsilon \in (0, 1/2)$, let $L = \log\left(\frac{1-\varepsilon}{\varepsilon}\right)$, and for $\alpha_1 > L$, let
\begin{align*}
\underline{\alpha}_0(\alpha_1) = \frac{1}{2\varepsilon}\log\left(\frac{\alpha_1 + L}{\alpha_1 - L}\right)
\end{align*}
be the lower bound on $\alpha_0$ given in Proposition~\ref{thm:heavi_approx}.
Then,
\begin{align*}
\lim_{\alpha_1 \to \infty} \alpha_1 \, \underline{\alpha}_0(\alpha_1) = \frac{1}{\varepsilon}\log\left(\frac{1-\varepsilon}{\varepsilon}\right).
\end{align*}
That is, $\underline{\alpha}_0(\alpha_1) \sim L/(\varepsilon\alpha_1)$ as $\alpha_1 \to \infty$.
\end{theorem}
In particular, when $\varepsilon = 0.1$, Theorem~\ref{thm:asymp_lower_bound} implies that 
$\lim_{\alpha_1 \to \infty} \alpha_1 \, \underline{\alpha}_0(\alpha_1) = 10\log(9)\approx 22$, which is consistent with the convergence observed in Figure~\ref{fig:ratio_lowerbound}.
See Appendix~\ref{app_asymp_lower_bound} for the proof of Theorem~\ref{thm:asymp_lower_bound}.

We now study the gradient behavior of the proposed HTAF. 
The gradient of HTAF can exhibit a heavy-tailed behavior: that is, it decays slowly for large-magnitude inputs while maintaining a large magnitude near zero.
This property is crucial for preventing premature gradient vanishing when approximating the Heaviside function.
The following theorem summarizes this nice property of HTAF.

\begin{theorem}
\label{thm:main_theorem}
    For any $\alpha_{0} \in \mathbb{R}_{+}$ and $\alpha_{1} \in \mathbb{R}_{+}$, we have the following:
    \begin{enumerate}[label=(\alph*)]
    \item For large $|x|$, we have 
    \begin{align*}
    \frac{\partial \text{HTAF}(x | \alpha_{0},\alpha_{1})}{\partial x}
    = C(\alpha_{0},\alpha_{1}) \exp(-2\alpha_{0}|x|),
    \end{align*} \label{eq:new(a)}
    where $C(\alpha_{0},\alpha_{1})$ is a constant that depends only on $\alpha_{0}$ and $\alpha_{1}$.
    \item We have 
    \begin{align*}
    {\partial \text{HTAF}(x|\alpha_{0},\alpha_{1}) \over \partial x} \bigg |_{x=0} &= {\alpha_{0}\alpha_{1} \over 4}. 
    \end{align*} 
    \label{eq:(a)}
    \item For any $\beta_{0} >0$ and any $\gamma \in (0,1)$, there exist $\alpha_{0}, \alpha_{1} > 0$ such that 
    \begin{align*}
    {\partial \text{HTAF}(x|\alpha_{0},\alpha_{1})\over \partial x}\bigg |_{x=0} =  {\partial \sigma(\beta_{0} x)\over \partial x}\bigg |_{x=0}
    \end{align*}
    and
    \begin{align*}
        \lim_{|x| \to \infty} \; e^{-r\beta_{0}|x|} \,
    {\partial \text{HTAF}(x|\alpha_{0},\alpha_{1}) / \partial x \over \partial \sigma_{\beta_{0}}(x) / \partial x} > 0 .
    \end{align*}
    \label{eq:(c)}
\end{enumerate}
\end{theorem}

Theorem \ref{thm:main_theorem} - \ref{eq:new(a)} implies that, in the tail region, the gradient of HTAF is governed mostly by $\alpha_{0}$.
In Figure \ref{fig_all_plots}-(b), we observe that as $\alpha_{0}$ increases, the gradient in the tail region decreases.
On the other hand, Theorem \ref{thm:main_theorem}-\ref{eq:(a)} implies that the slope of HTAF at the origin is determined by the product of the parameters $\alpha_{0}$ and $\alpha_{1}$.
This result implies that we can slow down the gradient decay by choosing a small $\alpha_0$, while maintaining a large gradient magnitude at zero by choosing a sufficiently large $\alpha_1$ such that $\alpha_0\alpha_1$ remains large.
To sum up, from \ref{eq:new(a)} and \ref{eq:(a)} of Theorem \ref{thm:main_theorem}, 
it follows that when $\alpha_{0}$ is chosen to be small but $\alpha_{1}$ to be sufficiently large, HTAF maintains a sufficiently large gradient near zero while alleviating gradient vanishing in the tail regions.
Theorem \ref{thm:main_theorem}-\ref{eq:(c)} shows that, under the same gradient at $x=0$ as Scaled Sigmoid, HTAF retains exponentially larger gradients in the tail region, which mitigates the vanishing gradient problem.

We compare the gradient behavior of HTAF with that of Scaled Sigmoid (SS), which is defined as $\sigma_{\beta_{0}}(x) := \sigma(\beta_{0}x)$.
The parameters are set to be $\alpha_{0}=1.5$, $\alpha_{1}=5.0$, and $\beta_{0}=7.5$, ensuring that both functions have identical derivatives at $x=0$.
Figure \ref{fig_all_plots}-(c) presents the ratio of the HTAF derivative magnitude to the SS derivative magnitude in the tail regions, \( |x| \in [2,3] \).
The result clearly shows that the gradient of HTAF is much larger than that of SS in the tail regions. 

Theorem \ref{thm:main_theorem} suggests that choosing a small $\alpha_{0}$ and large $\alpha_{1}$ reduces gradient vanishing 
(See Appendix \ref{app_grad_2}). 
However, an excessively small $\alpha_0$ may cause numerical issues, because maintaining a large gradient of HTAF at zero requires an excessively large $\alpha_1$, which can make the gradient of HTAF numerically unstable and effectively vanish.
The results of numerical studies for the choice of $\alpha_0$ and $\alpha_1$, which are presented in Appendix \ref{app_selct_alpha01}, suggest that extremely small values of $\alpha_{0}$ or extremely large values of $\alpha_{1}$ should be avoided, as they can lead to unstable training.

\paragraph{Applicability as a general-purpose activation function.}
Although HTAF is designed to approximate the Heaviside function, it can 
also replace smooth activations such as sigmoid and ReLU. In particular, a 
DNN with HTAF achieves a near-minimax optimal convergence rate for smooth 
target functions when $\log(\alpha_0\alpha_1)$ is bounded; see Appendix 
\ref{app:minmax}. 
It is known that DHNs have inferior approximation properties for smooth 
target functions compared to standard DNNs with smooth activations 
(\cite{kong2025expressivity}). 
Surprisingly, however, DNNs with HTAF perform well at approximating both smooth functions and DHNs.
One possible explanation is that the degree to which DNNs with HTAF activations approximate DHNs depends on the network architecture (e.g., the number of layers and the number of nodes per layer), which may lead to different asymptotic behaviors.
This issue would be beyond the scope of this paper, and we leave it for future work.

\subsection{Applications}
\subsubsection{Training Neural Networks with the Heaviside activation function}

While Heaviside activation functions enable binary representations that are computationally efficient during inference, they preclude the direct use of gradient-based optimization because their gradients are zero everywhere.
Existing approaches therefore often rely on surrogate gradients, which introduce a mismatch between the forward and backward passes, or on smooth activation-based approximations that remain susceptible to gradient vanishing.

The first application of HTAF is surrogate-free training of neural networks with the Heaviside activation function, including SNNs, BNNs, and DHNs. 
Specifically, we replace each Heaviside activation functions with HTAF during training, allowing the network to be optimized using a standard gradient descent algorithm.
After training, HTAF is replaced by the Heaviside function at the inference phase. 
This procedure preserves the learned weights and produces exact discrete activations, allowing the resulting model to benefit from the memory and computational efficiency of Heaviside-based inference without losing accuracy much.


\begin{figure}[t]
    \centering
    \includegraphics[width=0.95\linewidth]{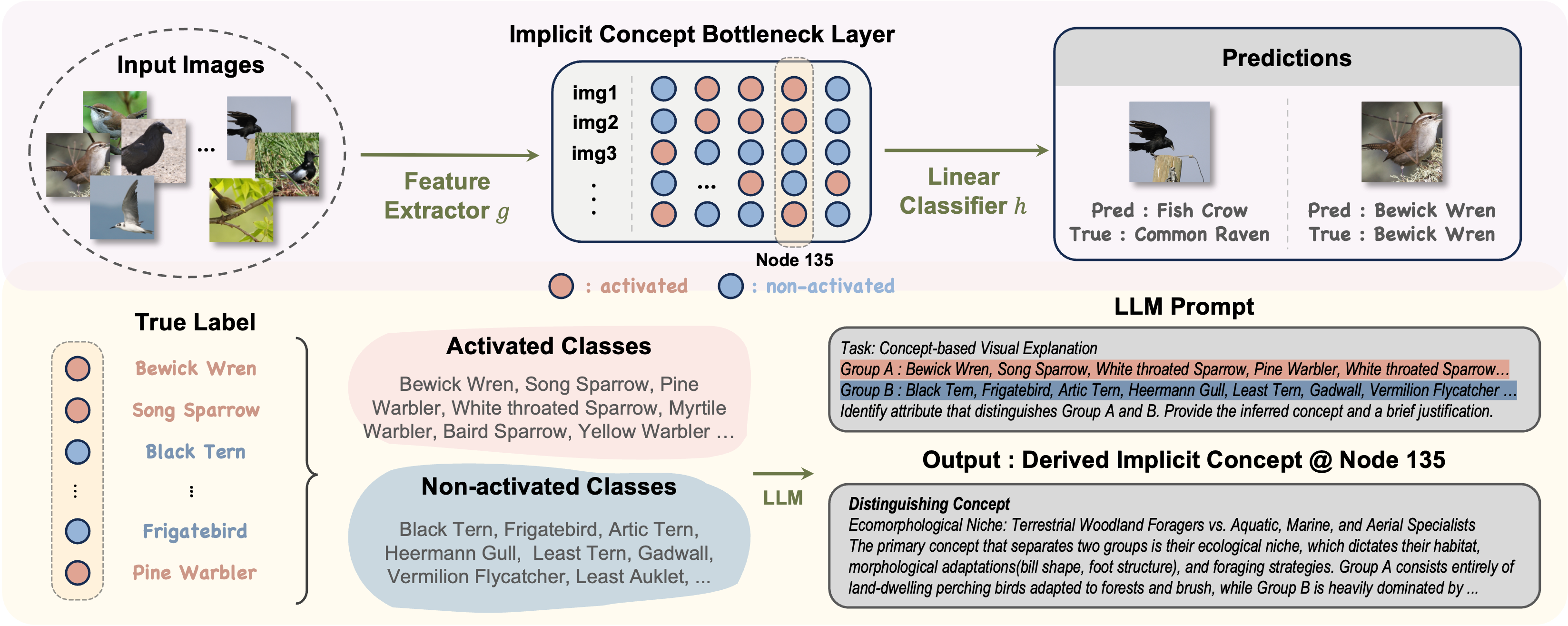}
    \caption{\textbf{Overview of Implicit Concept Bottleneck Models}.}
    \label{fig_overview}
    \vskip -0.5cm
\end{figure}

\subsubsection{Implicit Concept Bottleneck Models}
\label{sec:icbm-app}

Existing CBMs typically rely on either human-annotated or AI-annotated concepts. 
Human annotation is costly and labor-intensive, while AI-generated annotations can be unreliable.
We propose a new approach to obtaining concepts for CBMs: instead of first asking an AI system to define or annotate concepts, we first learn concepts directly from data and then ask AI to interpret their meanings.
This allows us to obtain concepts that are naturally relevant to the final prediction task.
We refer to these data-driven concepts as implicit concepts. 
For implicit concepts to be desirable, we require them to satisfy the following properties:
\begin{enumerate}
\item \textbf{Binarity}: implicit concepts should be binary, indicating the presence or absence of a given concept.
\item \textbf{Class specificity}: implicit concepts should be class-specific; that is, for images belonging to the same class, each implicit concept should tend to be consistently activated or deactivated.
\item \textbf{Completeness}: implicit concepts should be complete, in the sense that predictions made solely from implicit concepts should not suffer from a significant loss in accuracy.
\end{enumerate}

We propose Implicit Concept Bottleneck Models (ICBMs) to estimate implicit concepts. 
To be more specific, consider an arbitrary image classification model $f$, which can be decomposed into a feature extractor $g$ and a linear classifier $h$, i.e., $f(\mathbf{x}) = h(g(\mathbf{x}))$, where $\mathbf{x}$ is an input image. 
We define ICBM as a class of models where the activation function in the final layer of $g$ is replaced with HTAF during training to approximate the Heaviside function, and with the Heaviside function at inference to produce exact binary outputs. 
In this way, $g$ estimates implicit concepts.
Because HTAF enables stable training of ICBM with gradient-based optimization, ICBM retains prediction performance comparable to, or better than, that of conventional image models, as shown in Section \ref{sec:HTAF_perform}.

We now describe how each implicit concept can be interpreted once it is estimated by ICBMs.
Since each concept is class-specific, we divide the class labels into two groups—one for activation and the other for non-activation.
We then ask an LLM to explain a key feature that distinguishes these two groups.
If there are pre-defined concept candidates provided by domain experts or learned automatically from data (\cite{ghorbani2019towards}), we can ask an LLM to select one from the candidates or to generate a new concept. 
The resulting description can be interpreted as the semantic meaning of a given implicit concept.
An overview of ICBMs is presented in Figure \ref{fig_overview}.

Compared to CBM (\cite{koh2020concept}), ICBM has the following advantages.
First, ICBM does not require annotated concept labels for images, making it applicable to a much broader range of datasets.
Second, despite imposing this discrete bottleneck, ICBM does not suffer from a degradation in prediction performance, suggesting that the learned implicit concepts are complete.
Our empirical results in Section \ref{sec:class-spec} further show that the implicit concepts learned by ICBM are also class-specific.

\section{Experiments}

Comprehensive experimental results across a wide range of settings and tasks, along with ablation studies, applications to Large Language Models, and detailed descriptions of datasets, learning algorithms, and hyperparameter selection, are provided in Appendices \ref{app:detials_all_exp} to \ref{app_ICBM_results}.


\subsection{Deep Neural Networks with the Heaviside activation function}

\begin{wraptable}{r}{0.52\linewidth}
\vspace{-0.4cm}
\centering
\footnotesize
\caption{\textbf{Classification accuracies (\%) on CIFAR-10 and CIFAR-100 for SNNs, BNNs and DHNs.}}
\label{table:combined_image}

\begin{tabular}{ccccc}
\toprule
Model & Dataset & HTAF & STE & SS \\
\midrule
\multirow{2}{*}{SNN}
 & CIFAR-10  & 91.0\% & \textbf{91.2\%} & 90.6\% \\
 & CIFAR-100 & \textbf{66.9\%} & 63.9\% & 66.2\% \\
\midrule
\multirow{2}{*}{BNN}
 & CIFAR-10  & \textbf{75.8\%} & 75.3\%  & 60.6\% \\
 & CIFAR-100 & \textbf{46.1\%} & 45.2\%  & 31.3\% \\
\midrule
\multirow{2}{*}{DHN}
 & CIFAR-10  & \textbf{88.6\%} & 83.3\%  & 74.5\% \\
 & CIFAR-100 & \textbf{57.6\%} & 56.7\%  & 45.6\% \\
\bottomrule
\end{tabular}
\vspace{-0.3cm}
\end{wraptable}


We compare the prediction performance of SNNs, BNNs and DHNs using STE, HTAF and SS.
When using HTAF or SS, all Heaviside activation functions are replaced with these functions during training and reverted to the Heaviside function at inference.
For STE method, the backward gradients are computed using a smooth surrogate function during training.
For each model, we adopt backbone architectures as follows: Spiking Element-Wise ResNet (SEW-ResNet, \cite{fang2021deep}) for SNNs, Bi-Real Net (\cite{liu2018bi}) for BNNs, and WideResNet (\cite{zagoruyko2016wide}) for DHNs.
Table \ref{table:combined_image} shows the average prediction accuracy over five runs on CIFAR-10 and CIFAR-100 datasets for all models.

These results indicate that HTAF achieves prediction performance competitive with STE in learning SNNs, while consistently outperforming both STE and SS in BNNs and DHNs.
Overall, SS exhibits relatively poor performance in BNNs and DHNs, which is likely attributable to the gradient vanishing problem. Empirical evidence supporting this observation is provided in Appendix \ref{app_grad_1}.

Note that HTAF requires no special optimization techniques beyond standard gradient-based optimization. 
These results demonstrate the practical effectiveness of HTAF for training neural networks with Heaviside activation functions.
Additional experiments including comparison ANN-to-SNN method are in Appendix \ref{app:ann-to-snn} and the details of the experiments are in Appendix \ref{app:detials_all_exp}.

\subsection{Implicit Concept Bottleneck Models}

\begin{table*}[h]
\centering
\footnotesize
\caption{\textbf{Results of the average prediction accuracy with and without HTAF (i.e., using identity or ReLU) at the final layer of the feature extractor.}}
\vskip -0.2cm
\label{table_HTAF_predict}
\begin{tabular}{ccccccc}
\toprule
Dataset  & ResNet+ HTAF & ResNet & WideResNet+ HTAF & WideResNet & ViT+ HTAF & ViT \\ \midrule
\textsc{CIFAR-10} & 94.6\% & 94.6\% & 96.1\% & 95.7\% & 96.6\% & \textbf{96.7\%} \\
\textsc{CIFAR-100} & 80.8\% & 80.0\% & 79.5\% & 78.2\% & \textbf{89.0\%} & 88.6\%\\
\textsc{CUB-200} & 74.8\% & 74.5\% & 77.3\% & 77.5\% & \textbf{88.5\%} & 87.1\%\\
\textsc{Tiny Imagenet} & 72.2\%  & 72.1\%  & 85.4\% & 85.7\% & 86.6\% & \textbf{86.7}\% \\ 
\bottomrule
\end{tabular}
\end{table*}

\subsubsection{Prediction performance}
\label{sec:HTAF_perform}

In this section, we empirically demonstrate that HTAF can be effectively incorporated into various image classification models. 
We consider ResNet (\citep{he2016deep}), WideResNet (\citep{zagoruyko2016wide}), and ViT \citep{dosovitskiy2020image} as backbone image models.
We compare the prediction performance of the ICBM versions of these models to that of their original counterparts.
Here, the ICBM version refers to a model trained by replacing the activation function at the final layer of the feature extractor with HTAF, and evaluated at inference by replacing HTAF with the Heaviside function.

Table \ref{table_HTAF_predict} presents the average prediction accuracy over five trials for various architectures on \textsc{CIFAR-10}, \textsc{CIFAR-100}, \textsc{CUB-200}, and \textsc{Tiny-ImageNet} datasets.
Surprisingly, replacing the activation function in the last layer of the feature extractor with HTAF does not degrade prediction performance, even though the feature vectors are binary during inference.
In particular, on \textsc{CIFAR-100} and \textsc{CUB-200} datasets, we observe that applying HTAF even leads to improved prediction performance.
These results support that the implicit concepts estimated by ICBM are complete.
Note that CBMs (\citep{oikarinen2023label,yuksekgonul2022post,vandenhirtz2024stochastic}) typically suffer from degraded performance compared to conventional image models. 
Results of comparisons with other CBMs are provided in Appendix \ref{app:compare_CBM}.

\subsubsection{Evaluating class specificity of implicit concepts}
\label{sec:class-spec}

In this section, we evaluate whether implicit concepts are class-specific.
That is, we examine whether the activation values on each node are consistently either 0 or 1 for images belonging to the same class.
We use WideResNet as the backbone model and conduct experiments on CIFAR-10 dataset.

For images of the `Deer' class, we compute, for each node, the proportion of outputs equal to one.
As a baseline, we consider the original WideResNet and apply post hoc binarization to its activation values, converting them to zero or one.
The binarization threshold for each node is chosen to match the proportion of zeros and ones to those observed in ICBM.

Figure \ref{fig:hist-20} shows the histograms of these proportions for all methods on the `Deer' class.
Our method exhibits a strong concentration near 0 and 1, split almost evenly between the two, whereas the baseline does not.
This indicates that the implicit concepts learned by ICBM are activated in a class-specific manner.
Additional results for other classes are in Appendix \ref{app_class_specific}.

\begin{figure}[h]
    \centering
    \includegraphics[width=0.9\linewidth]{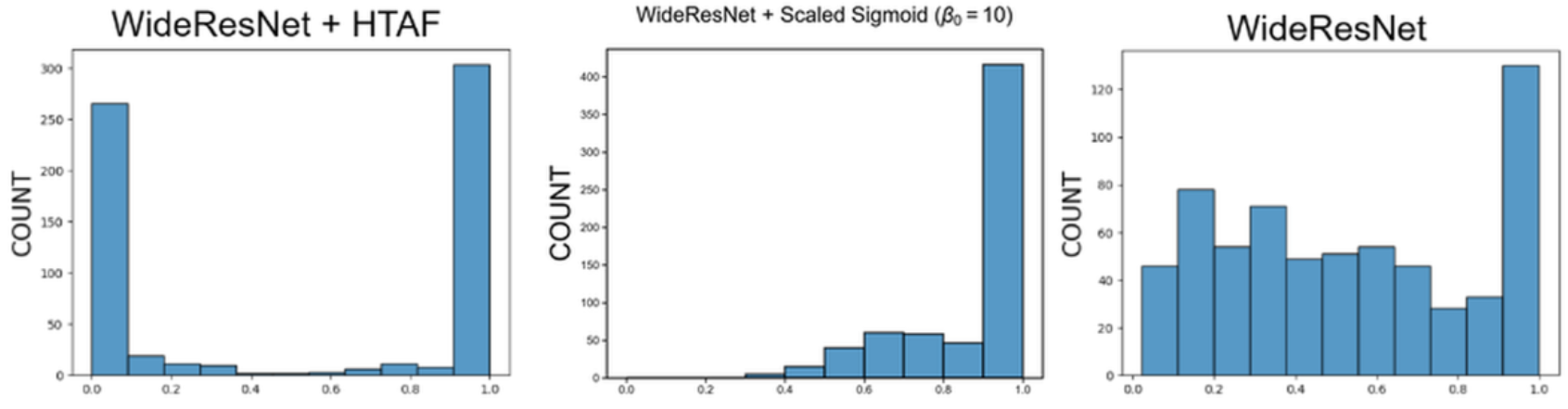}
    \vskip -0.15cm
    \caption{\textbf{Histograms of the proportions for the `Deer' class.}}
    \label{fig:hist-20}
\end{figure}

\subsubsection{Interpretation} 
\label{sec:interpret}
In this section, we illustrate how the learned implicit concepts can be interpreted using an LLM on \textsc{CUB-200} dataset with the ViT model.
For the \textit{Bewick's Wren} class, the top five important nodes based on the importance scores defined in Appendix \ref{app_def_score} are the 135th, 186th, 690th, 153rd, and 217th nodes. 
We analyze the semantic meaning of these nodes.

Figure \ref{fig:cub-node135} provides a visualization of the semantic properties captured by the 135th node in the final layer of the feature extractor.
The upper row shows representative images from the activated group, corresponding to samples with activation value one, while the lower row shows those from the unactivated group, corresponding to samples with activation value zero.

We identify the implicit concept associated with the 135th node using the following steps. 
First, we collect the class labels of images that activate the node and those that do not, forming two groups. 
We then ask an LLM to explain distinguishing features between these two groups. 
Specifically, we compile two class-name lists and ask the LLM to infer semantic meanings of the corresponding implicit concept using the following prompt. 
The LLM's responses are provided in Appendix \ref{app_LLM_responce}.

\begin{tcolorbox}[colback=gray!10, colframe=teal!60, title=Prompt]
\begin{scriptsize}
\begin{verbatim}
Task: Concept-based Visual Explanation

Group A: [Names of classes whose constituent images have an activation value of 1 at the node]
Group B: [Names of classes whose constituent images have an activation value of 0 at the node]

Identify the attribute that distinguishes Group A from Group B based on the two class lists. 
Provide the inferred concept and a brief justification.
\end{verbatim}
\end{scriptsize}
\end{tcolorbox}

\begin{figure*}
    \centering
    \begin{subfigure}[b]{0.68\textwidth}
        \centering
        \includegraphics[width=\textwidth]{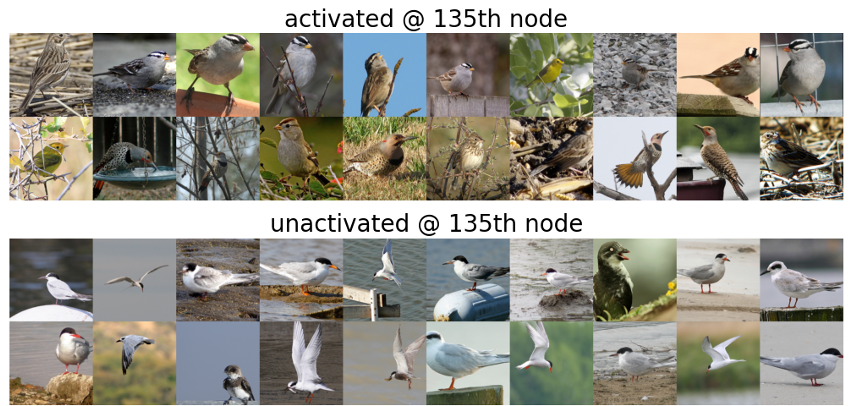}
        \caption{Activated and unactivated samples}
        \label{fig:cub-node135}
    \end{subfigure}
    \hfill
    \begin{subfigure}[b]{0.3\textwidth}
        \centering
        \includegraphics[width=\textwidth]{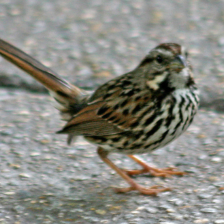}
        \caption{Misclassified instance}
        \label{fig:cub-miscls}
    \end{subfigure}
    \vspace{2mm}
    \caption{\textbf{Visual analysis of embedding space.} (a) The selected 20 images, showing activation value one (top) and zero (bottom) at the 135th node. (b) an image of \textit{Bewick Wren} which is misclassified.}
    \label{fig:cub-figs}
    \vskip -0.5cm
\end{figure*}

In response to the prompt, the LLM identifies the underlying concept of the node as \textit{`Ecomorphological Niche: Terrestrial Woodland Foragers vs. Aquatic, Marine, and Aerial Specialists'}.
This result is highly consistent with Figure \ref{fig:cub-node135}.
Notably, although the LLM is provided solely with textual class labels, the resulting explanation demonstrates a strong alignment with the visual characteristics of the actual images.
We repeat this process for the other 4 nodes and obtain the results shown in Table \ref{table_implicit_concept}.


\begin{table}[h]
\centering
\footnotesize
\caption{\textbf{Results of the estimated implicit concepts for top five important nodes}.}
\label{table_implicit_concept}
\begin{tabular}{cccc}
\toprule
Node index & Estimated implicit concept & Label & Importance score\\ \midrule
135 & Ecomorpohological Niche & 0: Aquatic / 1: Terrestrial & 0.1069\\ 
186 & Bill Morphology & 0: Stout or Conical / 1: Hooked or Spear-like  & 0.0941\\ 
690 & Habitats & 0: Dense Shrubland / 1: Open grassland & 0.0922\\ 
153 & Foraging mechanics & 0: Extraction / 1: Snatching   & 0.0879\\
217 & Visual Plumage Pigmentation & 0: Yellow \& Carotenoid / 1: Otherwise & 0.0840\\
\bottomrule
\end{tabular}%
\end{table}

We now investigate why a certain image is misclassified.
While most \textit{Bewick's Wren} images are correctly classified, the one shown in Figure \ref{fig:cub-miscls} is misclassified.
For correctly classified images, activation values of the five most important nodes (135, 186, 690, 153, and 217) are $(1, 1, 0, 0, 1)$, indicating that the implicit concepts associated with nodes 135, 186, and 217 are present, while those corresponding to nodes 690 and 153 are absent.
In contrast, for the misclassified image in Figure \ref{fig:cub-miscls}, ICBM predicts activation values for the important nodes as $(1, 0, 1, 1, 0)$, where nodes 135, 690, and 153 are active, and the others are not. 
Visually, the bird appears to exhibit features such as \textit{`stout and conical bill', `open grassland', and `yellow and carotenoid plumage'}.
As a result, ICBM misclassifies the image.

\section{Conclusion}

We proposed \emph{Heavy Tail Activation Function (HTAF)}, a novel activation function that stably approximates the Heaviside function under gradient-based optimization by preserving gradient magnitude.
We theoretically show that HTAF enables DNNs with Heaviside activation functions to learn effectively while alleviating the gradient vanishing problem in tail regions.

Based on the properties of HTAF, we further showed that HTAF provides an effective way to train DNNs with the Heaviside activation function using a standard gradient descent algorithm, without surrogate gradients, while significantly improving training stability and prediction performance.
Furthermore, we introduced \emph{Implicit Concept Bottleneck Models (ICBMs)}, which estimate implicit concepts using HTAF while maintaining prediction performance across various architectures and datasets.

Although tuning the parameters of HTAF can alleviate the gradient vanishing problem, it cannot eliminate it completely.
Addressing this limitation remains an important direction for future research.
In addition, an important direction for future work is to investigate principled methods for determining the number of concepts in ICBM, as well as the theoretical properties of DNNs employing HTAF.

\section*{Acknowledgments}
This work was supported by the National Research Foundation of Korea(NRF) grant funded by the Korea government(MSIT) (No. 2022R1A5A7083908), 
the Institute of Information \& communications Technology Planning \& Evaluation(IITP) grant funded by the Korea government(MSIT) (No. RS-2026-25522672, Development of Unified Reasoning Technology Mimicking Human Cognition for Hierarchical Understanding and Unbounded Problem Solving, NO.RS-2021-II211343, 
Artificial Intelligence Graduate School Program (Seoul National University)), 
the National Research Foundation of Korea(NRF) grant funded by the Korea government(MSIT) (RS-2025-00556079) 
and ERC grant A2B (grant agreement no. 101124751). 

\bibliography{references}
\bibliographystyle{plain}

\newpage
\appendix
\onecolumn

\newpage

\section*{Impact statement}

This paper presents research aimed at advancing the field of Machine Learning.
While our work may have broader societal implications, we do not identify any specific consequences that require particular discussion here.

\section{Proofs}

\label{app_proofs}

\subsection{Proof of Proposition \ref{thm:heavi_approx}}

\paragraph{Case of $x \geq \epsilon$.}
Since
\begin{align*}
\sigma(z) \geq 1- \epsilon
\quad \Leftrightarrow \quad
z \geq \log \bigg({1-\epsilon \over \epsilon}\bigg)
\end{align*}
and
\begin{align*}
\alpha_{1}\psi(\alpha_{0} x) &= \alpha_{1}\bigg( {\exp(2\alpha_{0}x) - 1 \over \exp(2\alpha_{0}x) + 1} \bigg ),
\end{align*}
it is sufficient to find ranges of $\alpha_0$ and $\alpha_1$ satisfying the followings:
\begin{align}
\alpha_{1}\bigg( {\exp(2\alpha_{0}x) - 1 \over \exp(2\alpha_{0}x) + 1} \bigg ) \geq \log \bigg({1-\epsilon \over \epsilon}\bigg).
\label{eq:bdd_main}
\end{align}

\eqref{eq:bdd_main} is satisfied whenever $\alpha_0$ and $\alpha_1$ satisfy the followings:
\begin{align*}
\alpha_{1} > \log \bigg({1-\epsilon \over \epsilon}\bigg) \quad \text{and} \quad \alpha_{0} \geq {1\over 2\epsilon}\log \bigg( {\alpha_{1} + \log\big({1-\epsilon \over \epsilon}\big) \over  {\alpha_{1} - \log\big({1-\epsilon \over \epsilon}\big)}} \bigg).
\end{align*}
\paragraph{Case of $x \leq -\epsilon$.}
The proof is identical to that of the case $x \geq \epsilon$ and is therefore omitted.

\qed

\subsection{Proof of Theorem \ref{thm:asymp_lower_bound}}
\label{app_asymp_lower_bound}

Let $t=L/\alpha_1$. Then, as $\alpha_1 \rightarrow \infty$, we have $t\rightarrow 0$ and
\begin{align*}
\alpha_1\underline{\alpha}_0(\alpha_1)
&=
\frac{\alpha_1}{2\epsilon}
\log\left(
\frac{1+t}{1-t}
\right)\\
&=
\frac{L}{\epsilon}
\frac{\log(1+t)-\log(1-t)}{2t}.
\end{align*}
Since
\begin{align*}
\lim_{t\rightarrow 0}
\frac{\log(1+t)-\log(1-t)}{2t}
=
1,
\end{align*}
we have
\begin{align*}
\lim_{\alpha_1\rightarrow\infty}
\alpha_1\underline{\alpha}_0(\alpha_1)
=
\frac{L}{\epsilon}
=
\frac{1}{\epsilon}
\log\left(\frac{1-\epsilon}{\epsilon}\right).
\end{align*}

\qed



\subsection{Proof of Theorem \ref{thm:main_theorem}}

\paragraph{Proof of (a).}

Direct calculation yields that
\begin{align*}
{\partial \text{HTAF}(x|\alpha_{0},\alpha_{1}) \over \partial x} &= \alpha_{1}\sigma(\alpha_{1}\psi(\alpha_{0}x))(1-\sigma(\alpha_{1}\psi(\alpha_{0}x))) {\partial \psi(\alpha_{0}x) \over \partial x} \\
&= \alpha_{1}\sigma(\alpha_{1}\psi(\alpha_{0}x))(1-\sigma(\alpha_{1}\psi(\alpha_{0}x)))  \bigg(\alpha_{0}(1-\{\psi(\alpha_{0}x)\}^{2})\bigg).
\end{align*}

That is, for large $|x|$, we have
\begin{align*}
{\partial \text{HTAF}(x|\alpha_{0},\alpha_{1}) \over \partial x} & = C(\alpha_{0},\alpha_{1})(1-\{\psi(\alpha_{0}x)\}^{2}) \\
&= C(\alpha_{0},\alpha_{1})\exp(-2\alpha_{0}|x|)
\end{align*}
for some constant $C(\alpha_{0},\alpha_{1})$ that depends on $\alpha_{0}$ and $\alpha_{1}$.

\qed

\paragraph{Proof of (b).}

From proof of (a), we have
\begin{align*}
{\partial \text{HTAF}(x|\alpha_{0},\alpha_{1}) \over \partial x} \bigg |_{x=0} &= {\alpha_{0}\alpha_{1}\over 4}. 
\end{align*}
\qed

\paragraph{Proof of (c).}
The proof follows directly from the proof of (a) in Theorem \ref{thm:main_theorem} by choosing the hyperparameters $\alpha_{0}$ and $\alpha_{1}$ such that $\alpha_{0}\alpha_{1} = \beta$ and $\alpha_{0} \leq \frac{(1-r)\beta}{2}$.

\qed

\newpage

\section{Theoretical properties of Deep Neural Networks with HTAF}
\label{app:minmax}
In this section, we study the minimax optimality of DNNs with HTAF for regression tasks. 
Here, DNNs with HTAF refers to DNNs that uses HTAF as its activation function.
We first introduce the notation used throughout this section.
For simplicity, we assume $\alpha_0, \alpha_1 \geq 2$.

\subsection{Notations}

\paragraph{Basic Notation.}
Let $\mathbf{m} = (m_{1}, \ldots, m_{d}) \in \mathbb{N}_{0}^{d}$ be a multi-index, where
$\mathbb{N}_{0} := \mathbb{N} \cup \{0\}$.
For a function $f : [0,1]^d \to \mathbb{R}$, we define the sup-norm as
$\|f\|_{\infty} := \sup_{\mathbf{x} \in [0,1]^d} |f(\mathbf{x})|$ and for $1 \leq p < \infty$, we write $\Vert f \Vert_{p,\mu} := (\int_{\mathcal{X}}f(\mathbf{x})^{p}\mu(d\mathbf{x}))^{1/p}$.
For a given vector $\mathbf{x} \in \mathbb{R}^{d}$, we denote 
the $L_1$ norm by
$|\mathbf{x}|_{1} := \sum_{1 \le j \le d} |x_j|$,
the infinity norm by
$|\mathbf{x}|_{\infty} := \max_{1 \le j \le d} |x_j|$ and the zero norm by $|\mathbf{x}|_{0} := \sum_{j=1}^{d}\mathbb{I}(x_{j}\neq 0)$.
For $\mathbf{m} \in \mathbb{N}_{0}^{d}$, the mixed partial derivative of order $\mathbf{m}$ is defined as 
$$
\partial^{\mathbf{m}} f := \frac{\partial^{|\mathbf{m}|_1} f}{\partial x_{1}^{m_{1}} \cdots \partial x_{d}^{m_{d}}}.
$$
Finally, let $\mathcal{C}_d^{m}$ denote the class of functions on $[0,1]^d$ that are $m$ times continuously differentiable.
Here, the partial derivatives of order $\mathbf{m}$ with $|\mathbf{m}| \leq m$ are continuous.
We denote DNNs using the ReLU activation function as deep ReLU networks.

\paragraph{Definition of the H\"older space.}
For a smoothness $\beta>0$, we write $\lfloor \beta \rfloor$ for the greatest integer less than $\beta$. 
The H\"older space $\mathcal{H}_d^{\beta}$ is defined as 
\begin{align}
\mathcal{H}_d^{\beta} := \big\{f \in \mathcal{C}_d^{\lfloor \beta \rfloor} : \Vert f \Vert_{\mathcal{H}_d^{\beta}} < \infty \big\}, 
\end{align}
where 
\begin{align}
\Vert f \Vert_{\mathcal{H}_d^{\beta}} := \sum_{\mathbf{m} \in \mathbb{N}_{0}^{d}: |\mathbf{m}| \leq \lfloor \beta \rfloor } \Vert  \partial^{\mathbf{m}}f \Vert_{\infty} + \sum_{\mathbf{m} \in \mathbb{N}_{0}^{d}: |\mathbf{m}| = \lfloor \beta \rfloor } \sup_{\mathbf{x}_{1},\mathbf{x}_{2} \in [0,1]^d, \mathbf{x}_{1} \neq \mathbf{x}_{2}} { |\partial^{\mathbf{m}}f(\mathbf{x}_{1}) - \partial^{\mathbf{m}}f(\mathbf{x}_{2})| \over |\mathbf{x}_{1}-\mathbf{x}_{2}|_{\infty}^{\beta - \lfloor \beta \rfloor} }.
\end{align}
For example, $\mathcal{H}_d^{1}$ corresponds to the space of bounded Lipschitz continuous functions.
For $R>0$, we defined the H\"older space $\mathcal{H}^{\beta}_d(R)$ as
\begin{align}
\mathcal{H}^{\beta}_d(R) := \big\{ f \in \mathcal{H}_d^{\beta}, \Vert f \Vert_{\mathcal{H}_d^{\beta}} \leq R \big\}.
\end{align}

\paragraph{Definition of DNNs with HTAF.}
Let $\text{HTAF}_{\alpha_0, \alpha_1} := \text{HTAF}(\cdot|\alpha_{0},\alpha_{1})$.
A DNN with HTAF and $L \in \mathbb{N}$ layers, $n_{\ell} \in \mathbb{N}$ many nodes at the $\ell$-th hidden layer for $\ell=1,...,L$, input dimension $n_{0}=d$, output dimension $n_{L+1}=1$ and $\text{HTAF}_{\alpha_0, \alpha_1}$ is express as
\begin{align}
f^{\text{HTAF}}_{\theta, \alpha_0, \alpha_1}(\mathbf{x}) := A_{L+1}\circ \text{HTAF}_{\alpha_0, \alpha_1} \circ A_{L} \circ \cdots \circ A_2 \circ \text{HTAF}_{\alpha_0, \alpha_1} \circ A_{1}(\mathbf{x}),
\end{align}
where 
$A_{\ell}:\mathbb{R}^{n_{\ell-1}}\xrightarrow{} \mathbb{R}^{n_{\ell}}$ is an affine map defined by $A_{\ell}(\mathbf{x})=\mathbf{W}_{\ell}\mathbf{x}+\mathbf{b}_{\ell}$.
Here, $\mathbf{W}_{\ell} \in \mathbb{R}^{n_{\ell}\times n_{\ell-1}}$ is a weight matrix and $\mathbf{b} \in \mathbb{R}^{n_{\ell}}$ is a bias vector.
Let $\theta$ be the set of all parameters including weight matrices and bias vectors, i.e., $\theta := \big( \mathbf{W}_{1},\mathbf{b}_{1},...,\mathbf{W}_{L+1},\mathbf{b}_{L+1}\big)$. 
For a given parameter $\theta$, let $L(\theta)$ denote the number of hidden layers (depth), and let $n_{\max}(\theta)$ denote the maximum number of hidden units per layer (width).
Note that the zero norm and infinity norm of $\theta$ is expressed as
\begin{align}
|\theta|_{0} &= \sum_{\ell=1}^{L+1}\big( |\text{vec}(\mathbf{W}_{\ell})|_{0} + |\mathbf{b}_{\ell}|_{0} \big),\\
|\theta|_{\infty} &= \max \bigg\{ \max_{\ell \in [L+1]}|\text{vec}(\mathbf{W}_{\ell})|_{\infty} , \max_{\ell \in [L+1]}|\mathbf{b}_{\ell}|_{\infty} \bigg\},
\end{align}
where $\mathrm{vec}(\mathbf{W}_{\ell})$ denotes the vectorization of $\mathbf{W}_{\ell}$ obtained by concatenating its column vectors.
Finally, we define the function spaces $\mathcal{F}_{\alpha_0, \alpha_1}(L,N)$ and $\mathcal{F}_{\alpha_0, \alpha_1}(L,N,S,B)$ of DNNs with HTAF by
\begin{align}
\mathcal{F}_{\alpha_0, \alpha_1}(L,N) &:= \big \{ f^{\text{HTAF}}_{\theta, \alpha_0, \alpha_1} : L(\theta) \leq L, n_{\max}(\theta) \leq N \}, \\
\mathcal{F}_{\alpha_0, \alpha_1}(L,N,S,B) &:= \big \{ f^{\text{HTAF}}_{\theta, \alpha_0, \alpha_1} : L(\theta) \leq L, n_{\max}(\theta) \leq N, |\theta|_{0} \leq S, |\theta|_{\infty} \leq B\}.
\end{align}
\newpage

\subsection{Expressivity}

In this section, we investigate the expressive power of DNNs with HTAF.
First, we show that DNNs with HTAF can approximate any deep ReLU networks with arbitrary precision, provided that the network is a constant factor larger.

\begin{theorem} \label{theorem_approx_relu}
    For any deep ReLU network $f^{\text{ReLU}}$ with $L$ hidden layers and $N$ hidden units per layer, and any $\alpha_0, \alpha_1, \epsilon > 0$, there exists a DNN with HTAF
    \[
    f^{\text{HTAF}}_{\theta, \alpha_0, \alpha_1} \in \mathcal{F}_{\alpha_0, \alpha_1}(2L,3N)
    \]
    satisfying
    \[
    \|f^{\text{HTAF}}_{\theta, \alpha_0, \alpha_1} - f^{\text{ReLU}}\|_{\infty} \le \epsilon.
    \]    
\end{theorem}
\begin{proof}
Since $\text{HTAF}_{\alpha_0, \alpha_1}$ is twice differentiable and non-affine, there exists $x_0 \in \mathbb{R}$ such that $\text{HTAF}^{(2)}_{\alpha_0, \alpha_1}(x_0) \neq 0$.
We get the assertion by applying Theorem 1 of \cite{zhang2024deep}.  
\end{proof}

The approximation capabilities of deep ReLU networks have been widely studied across various settings (\cite{jiao2023deep, 10.1214/20-AOS2034, kong2025posterior, lu2021deep, petersen2018optimal, schmidt2019deep, 10.1214/19-AOS1875, wang2018exponential, yang2025optimal, YAROTSKY2017103}).
An implication of Theorem~\ref{theorem_approx_relu} is that the expressivity properties of deep ReLU networks carry over directly to DNNs with HTAF.

Interestingly, DHNs, despite being well approximated by DNNs with HTAF with sufficiently large $\alpha_0$ and $\alpha_1$, exhibit limited expressivity. 
For example, Theorem~1 of \cite{kong2025expressivity} shows that, for any non-constant target function, achieving an approximation error $ \leq \epsilon$ requires the DHN to have at least $\gtrsim 1/\epsilon$ neurons in the first hidden layer, regardless of how deep the model is.
On the other hand, combining Theorem~\ref{theorem_approx_relu} with Theorem~6 of \cite{wang2018exponential}, it follows that for any analytic function $f_0$, there exists a DNN with HTAF of depth $\lesssim \log^2(1/\epsilon)$ and constant width achieving an approximation error $\leq \epsilon$.
This shows that DNNs with HTAF not only help train DHNs but can also serve as a proxy for them without suffering from limited expressivity.

Indeed, $\alpha_0$ and $\alpha_1$ neither affect the depth and width required to achieve a given approximation error nor influence the scale of the parameters. 
The following theorem provides such an example and will be used to derive the statistical convergence rate later.

\begin{theorem} \label{theorem_approx_holder}
There exist positive constants $L_{0}, N_{0}, S_0$ and $B_0$ depending only on $(d, \beta, R)$, such that for any $f_{0} \in \mathcal{H}^{\beta}_d(R)$ and $\epsilon > 0$, there exists a DNN with HTAF
\[
f^{\text{HTAF}}_{\theta, \alpha_0, \alpha_1} \in \mathcal{F}_{\alpha_0, \alpha_1}\bigl(L_{0}\log(1/\epsilon),\; N_{0}\epsilon^{-d/\beta},\; S_{0}\epsilon^{-d/\beta},\; B_{0} \epsilon^{-4(d/\beta+1)}\bigr)
\]
satisfying
\[
\|f^{\text{HTAF}}_{\theta, \alpha_0, \alpha_1} - f_{0}\|_{\infty} \le \epsilon.
\]
\end{theorem}

\begin{proof}
We apply Theorem~1 of \cite{ohn2019smooth}, keeping track of the dependence of the implicit constants on $\alpha_0$ and $\alpha_1$. 
It suffices to consider the construction of the square function $x \mapsto x^2$ (Lemma~A3.(a) of \cite{ohn2019smooth}), as the remaining part of the proof depends only on this structure.

Suppose that an activation function $\phi$ is three times continuously differentiable. 
Additionally, assume that there exists an interval $(a,b)$ on which $\phi$ has bounded derivatives up to order three, and that there exists $t \in (a,b)$ such that both $\phi'(t)$ and $\phi''(t)$ are non-zero.
For sufficiently large $K$ with $\min(t-a, b-t) > 2/K$, 
\cite{ohn2019smooth} considered a deep neural network with $1$ hidden layer, $3$ hidden neurons and the activation function $\phi$, 
\begin{align}
    f_{\phi}(x) := \sum_{k=0}^2(-1)^{k-1} \frac{K^2}{\phi^{\prime \prime}(t)}\binom{2}{k} \phi\left(\frac{k}{K} x+t\right). \label{tmp_2}
\end{align}
Then, it holds for some constant $c>0$,
\begin{align}
    \sup_{x \in [0,1]} \big|f_{\phi}(x) - x^2\big| \leq c \, \frac{\sup_{z \in (a,b)} |\phi'''(z)|}{K |\phi''(t)|}.
    \label{tmp_1}
\end{align} 

For $s \in [0,1]$,
define
\[
u := \psi(\alpha_{0} s), \qquad p := \sigma(\alpha_1 u) = \text{HTAF}(s|\alpha_{0},\alpha_{1}).
\]
Then, we have 
\begin{align*}
    \text{HTAF}^{(1)}(s|\alpha_{0},\alpha_{1})
    &= \alpha_0 \alpha_1 (1-u^2) p (1-p), \\ \text{HTAF}^{(2)}(s|\alpha_{0},\alpha_{1})
    &= \alpha_0^2 \alpha_1 (1-u^2) p (1-p) [-2u + \alpha_1 (1-2p)(1-u^2)]\\
    \text{HTAF}^{(3)}(s|\alpha_{0},\alpha_{1})
    &= \alpha_0^3 \alpha_1 (1-u^2) p (1-p) \Big[
6u^2 - 2
- 6\alpha_1 u (1-2p)(1-u^2)\\
& \hspace{4.5cm}+ \alpha_1^2 (1-6p+6p^2)(1-u^2)^2
\Big].
\end{align*}
In particular, for any sufficiently large $\alpha_0$ and $\alpha_1$, if we put $s = s_0 := \psi^{-1}(\alpha_1^{-1})/\alpha_0$, then $u = 1/\alpha_1$ and $p = (1+e^{-1})^{-1}$. 
Hence,
\[
\text{HTAF}^{(2)}\Big(\frac{\psi^{-1}(\alpha_1^{-1})}{\alpha_0} \Big|\alpha_{0},\alpha_{1}\Big)
    \gtrsim \alpha_0^2 \alpha_1^2.
\]
On the other hand, as both $\psi(\cdot)$ and $\sigma$ are bounded function, we have
$\text{HTAF}^{(3)}(x|\alpha_{0},\alpha_{1}) \lesssim \alpha_0^3 \alpha_1^3$.

We plug in $\phi = \text{HTAF}_{\alpha_{0},\alpha_{1}}$ and 
$t=s_0$ into \eqref{tmp_2} and \eqref{tmp_1}. 
Then, for
\[
f^{\text{HTAF}}_{\xi, \alpha_0, \alpha_1} := \sum_{k=0}^2(-1)^{k-1} \frac{K^2}{\phi^{\prime \prime}(t)}\binom{2}{k} \phi\left(\frac{k}{K} x+s_0\right).
\]
we get
\begin{align*}
    \sup_{x \in [0,1]} \big|f^{\text{HTAF}}_{\xi, \alpha_0, \alpha_1} - x^2\big| \lesssim \frac{\alpha_0 \alpha_1}{K}.
\end{align*} 

Observe that the magnitudes of the parameters in $f^{\text{HTAF}}_{\xi, \alpha_0, \alpha_1}$ are bounded by $O(K^2/(\alpha_0 \alpha_1)^2)$. 
Recall that $K$ can be chosen arbitrary large.
By defining $M = K/(\alpha_0 \alpha_1)$, we obtain that for any sufficiently large $M$, there exists $f^{\text{HTAF}}_{\xi, \alpha_0, \alpha_1}$ with $L(\xi) =1$, $N(\xi) = 3$ and $|\xi|_{\infty} = O(M^2)$ such that  
\begin{align*}
    \sup_{x \in [0,1]} \big|f^{\text{HTAF}}_{\xi, \alpha_0, \alpha_1} - x^2\big| \lesssim \frac{1}{M},
\end{align*} 
where the implicit constants do not depend on $\alpha_0$ or $\alpha_1$. 
Hence, we recover the same conclusion as in Lemma~A3.(a) of \cite{ohn2019smooth}, and consequently the same assertion as in Theorem~1 of \cite{ohn2019smooth}.

\end{proof}

\subsection{Statistical convergence rate}

In this section, we investigate statistical convergence rates of DNNs with HTAF in a nonparametric regression setting.
Building upon the approximation results established in the previous section, we show that DNN estimators with HTAF achieves the near-minimax optimal convergence rate over the H\"older class, under mild assumptions on $\alpha_0$ and $\alpha_1$.
This result demonstrates that HTAF preserves not only the expressive power but also the statistical optimality of deep neural networks.

For a given $f_{0}$, we consider a nonparametric regression model where the input variable $\mathbf{X}$ follows a distribution $\mathbb{P}_{\mathbf{X}}$, and the response variable satisfies
\begin{align}
Y|\mathbf{X}=\mathbf{x} \sim N(f_{0}(\mathbf{x}),1).
\end{align}
Let $(\mathbf{X}_{1},Y_{1}),...,(\mathbf{X}_{n},Y_{n})$ be independent copies of $(\mathbf{X},Y)$.
Theorem \ref{theorem_statistical_convergence} shows that the DNN estimators with HTAF achieves the minimax optimal rate, and that the dependence on $\alpha_{0}$ and $\alpha_{1}$ appears only through logarithmic factors.

\begin{theorem} \label{theorem_statistical_convergence}
For any $\kappa, R, \alpha_0, \alpha_1 >0$, there exist positive constants $L_{0}, N_{0}, S_{0}$, and $B_{0}$ such that the DNN with HTAF estimator
\[
\hat{f}_{n} \in \arg\min_{f \in \mathcal{F}_{n}} \sum_{i=1}^{n}(Y_{i}-f(\mathbf{X}_{i}))^{2}
\]
with
\[
\mathcal{F}_{n}
:=
\left\{
f^{\text{HTAF}}_{\xi, \alpha_0, \alpha_1} \in
\mathcal{F}\bigl(L_{0}\log n,\,
N_{0}n^{\frac{d}{2\beta + d}},\,
S_{0}n^{\frac{d}{2\beta + d}}\log n,\,
B_{0}n^{\kappa}\bigr)
:
\|f^{\text{HTAF}}_{\xi, \alpha_0, \alpha_1}\|_{\infty} \leq 2R
\right\}
\]
satisfies
\[
\sup_{f_{0} \in \mathcal{H}^{\beta}_d(R)}
\mathbb{E}_{n}\!\left[
\| \hat{f}_{n} - f_{0} \|_{2,\mathbb{P}_{\mathbf{X}}}^{2}
\right]
\leq
C n^{-\frac{2\beta}{2\beta + d}} \Big(\log^3 n + \log^2 n \log(\alpha_0 \alpha_1)\Big)
\]
for some constant $C>0$.
\end{theorem}

\begin{proof}
Since $\text{HTAF}_{\alpha_0, \alpha_1}$ is $(\alpha_0 \alpha_1)$-Lipschitz continuous, Proposition~1 of \cite{ohn2019smooth} yields
\begin{align}
    \log \mathcal{N}\left(\frac{1}{n}, \mathcal{F}_{n}, \|\cdot\|_{\infty}\right) 
    \lesssim  n^{\frac{d}{2\beta + d}}\Big(\log^3 n + \log^2 n \log(\alpha_0 \alpha_1)\Big). \label{tmp_3}
\end{align}
We obtain the desired assertion by applying Lemma~4 of \cite{10.1214/19-AOS1875} together with Theorem \ref{theorem_approx_holder} and \eqref{tmp_3}.
\end{proof}

\newpage

\section{Approximation of a DNN with the Heaviside activation function by a DNN with HTAF activation}

\label{app_dhn_approx}

In this section, we prove that a DHN can be approximated by a DNN with HTAF activation as long as
the Heaviside function is well approximated by HTAF.
For notational simplicity, we only consider a DHN with one hidden layer. The proof could be extended for multi-layer DNNs without much  
difficulty.

Let $\sigma_{H}(x)$ be the Heaviside function (i.e. $\sigma_H(x)=\mathbb{I}(x\ge 0)$).
Let $f:\mathbb{R}^d \rightarrow \mathbb{R}$ be a DNN with the Heaviside activation function given as
$$f(\mathbf{x})=W_1^\top \sigma_{H} (\mathbf{W}_0 \mathbf{x}+\mathbf{b}),$$
where $W_1$ is a $n_1$ dimensional vector, $\mathbf{W}$ is an $n_1\times d$ dimensional matrix and
$\mathbf{b}$ is an $n_1$ dimensional vector. 
Let $g(\mathbf{x})$ be the DNN obtained by replacing the Heaviside activation in $f$ by HTAF activation. 
That is
$$g(\mathbf{x})=W_1^\top \text{HTAF}(\mathbf{W}_0 \mathbf{x}+\mathbf{b}|\alpha_0\alpha_1).$$
Let $h_j(\mathbf{x})$ be the $j$ entry of $\mathbf{W}_0 \mathbf{x}+\mathbf{b}$ for $j =1,...,n_{1}$.

\begin{theorem}
\label{app_dhn_dnnhtaf}
For given random vector $\mathbf{X}\sim \mathbb{P}_{\mathbf{X}},$
we assume the following regularity conditions:
\begin{itemize}
 \item[(C)] For $\epsilon \in (0,1/2),$ there exists a constant $C>0$ such that
$$\mathbb{P}_{\mathbf{X}} \{  |h_j(\mathbf{X})| \in (-\epsilon,\epsilon)\}\le C\epsilon$$
for all $j\in [n_1]$.
\end{itemize}
Then, the following holds.
$$\mathbb{E}_{\mathbf{X}}[|f(\mathbf{X})-g(\mathbf{X})|] \le  n_1 \|W_1\|_\infty (1+C)\epsilon.$$ 
\end{theorem}
\begin{proof}
Suppose that $\alpha_0$ and $\alpha_1$ satisfies the conditions in Proposition \ref{thm:heavi_approx}.
Then, we have
$$|f(\mathbf{x})-g(\mathbf{x})|\le  \|W_1\|_{\infty} \sum_{j=1}^{n_1} \left|\sigma_H(h_j(\mathbf{x}))-\text{HTAF}(h_j(\mathbf{x})|\alpha_0,\alpha_1)\right|.$$
Condition (C) and Proposition \ref{thm:heavi_approx} imply that
$$\mathbb{E}_{\mathbf{X}}[\left|\sigma_H(h_j(\mathbf{x}))-\text{HTAF}(h_j(\mathbf{x})|\alpha_0,\alpha_1)\right|]
\le (1+C)\epsilon,$$ and thus we conclude that
$$\mathbb{E}_{\mathbf{X}} [|f(\mathbf{X})-g(\mathbf{X})|] \le  n_1 \|W_1\|_\infty (1+C)\epsilon.$$ 
\end{proof}

Note that Condition (C) is mild and can be satisfied under reasonable assumptions on $\mathbb{P}_{\mathbf{X}}$, $\mathbf{W}_{0}$, and $\mathbf{b}$.


\newpage

\section{Details of the experiments}

\label{app:detials_all_exp}

\subsection{Data description}

\label{app:data_descp}

\begin{table}[h]
\centering
\footnotesize
\caption{\textbf{Descriptions for datasets}.}
\label{table:data_descp}
\begin{tabular}{ccccc}
\toprule
Dataset & Data size  & \# of feature & Type & Task \\ \midrule
\textsc{CIFAR-10} (\citep{krizhevsky2009learning}) & 60,000 & -- & Image & Classification \\ 
\textsc{CIFAR-100} (\citep{krizhevsky2009learning}) & 60,000 & -- & Image & Classification \\ 
\textsc{CUB-200} (\citep{WahCUB_200_2011}) & 11,788 & -- & Image & Classification \\ 
\textsc{Tiny-ImageNet} (\citep{le2015tiny})& 100,000 & -- & Image & Classification \\ \midrule
\textsc{Wine} (\citep{wine_quality})& 4,898 & 11 &   Tabular & Regression \\ 
\textsc{Abalone} (\citep{abalone}) & 4,178 & 8 & Tabular & Regression \\ 
\textsc{Churn} (\citep{churn}) & 7,043 & 20 &   Tabular & Classification \\ 
\textsc{FICO} (\citep{fico}) & 10,459 & 23 &   Tabular & Classification \\ \midrule
\textsc{Winogrande-small} (\citep{sakaguchi2021winogrande})& 1,907 & -- & Language & Classification  \\ 
\textsc{Winogrande-medium} (\citep{sakaguchi2021winogrande})& 3,823 & -- & Language & Classification  \\ 
\textsc{OpenBookQA} (\cite{mihaylov2018can}) & 5,456 & -- & Language & Classification \\ 
\bottomrule
\end{tabular}
\end{table}

Table \ref{table:data_descp} presents the descriptions of the datasets used in our experiments.

\subsection{Training details}

We experiment with both Adam and SGD optimizers and report the best-performing configuration.
The learning rate is selected via hyperparameter tuning from the candidate set $\{10^{-1}, 10^{-2}, 10^{-3}, 10^{-4}\}$.
The batch size is set to the largest value that can be accommodated within the available computational resources.
We standardize the image data independently for each RGB channel.
Additionally, all images were distributed as train/test sets. 
The train set was further split into train/validation datasets: the train set was used for model training, the validation set for hyperparameter tuning, and the test set for evaluating the final prediction performance.
For tabular data, we use min-max scaling for normalization.

\subsection{Computational resources}

All experiments were conducted using RTX 3090 and RTX 4090 GPUs with 24 GB VRAM, as well as RTX PRO 6000 GPUs with 96 GB VRAM.

\subsection{Experimental Details for Deep Neural Networks with the Heaviside activation function}

\label{app:details_heavi}

\paragraph{Details for Spiking Neural Networks experiments.}
\label{app:details_snn}

We use Spike-Element-Wise (SEW) ResNet as the backbone SNN model and implement it with minor modifications based on the official code available at \url{https://github.com/fangwei123456/Spike-Element-Wise-ResNet}.
For STE method, we consider multiple surrogate gradients, including the sigmoid, hyperbolic tangent, and rectangular functions, and report the best prediction performance among them. These surrogate gradients are defined as follows:
\begin{equation}
g_{\text{sigmoid}}(x) = \sigma(x)\left(1 - \sigma(x)\right),\quad g_{\text{tanh}}(x) = 1 - \tanh^2(x),\quad g_{\text{rect}}(x) =
\begin{cases}
1 & \text{if } |x| < 1 \\
0 & \text{otherwise}
\end{cases}
\end{equation}

For the parameters $\alpha_{0}$ and $\alpha_{1}$ in (\ref{eq:HTAF}), we select the optimal values from the candidate set $\{(0.1,220) ,(0.5, 44), (1, 22), (2, 11), (3, 7.3), (5, 4.4)\}$.
For $\beta_{0}$ in Scaled Sigmoid, we select the optimal value from the candidate set $\{10, 22, 50, 100\}$.

\paragraph{Details for Binary Neural Networks experiments.}

We use Bi-Real Net (\cite{liu2018bi}) as the backbone BNN model, and implement it with minor modifications based on the official code provided in \url{https://github.com/liuzechun/Bi-Real-net}.
For STE method, we use the same surrogate gradient as in \cite{liu2018bi}.
For $\alpha_{0}, $$\alpha_{1}$ and $\beta_{0}$, we select the optimal values from the candidate set in the SNN experiments.

\paragraph{Details for Deep Heaviside neural Networks experiments.}

We use WideResNet (\cite{zagoruyko2016wide}) as the backbone image model and implement it with minor modifications based on the official code available at \url{https://github.com/google-research/augmix}.
For STE method, we use the same surrogate gradients as those considered in the SNN experiments.
For $\alpha_{0}$, $\alpha_{1}$ and $\beta_{0}$, we choose optimal parameters from the candidate set in the experiment of SNNs.

\subsection{Details for ICBM experiments}

In this section, we describe the details of experiments for image datasets.
For $\alpha_{0}$, $\alpha_{1}$ and $\beta_{0}$, we choose optimal parameters from the candidate set in Appendix \ref{app:details_snn}.
We use GPT-5 (\cite{singh2025openai}) as a large language model to interpret the implicit concepts estimated by ICBM.
The details of architectures for the backbone image models are summarized below.

\paragraph{ResNet.} 
We use ResNet-18 architecture for all experiments, with different pretrained weights depending on the dataset. Specifically, for \textsc{CIFAR-10}, we use a model from \url{https://huggingface.co/edadaltocg/ResNet18_cifar10}; for \textsc{CIFAR-100} and \textsc{Tiny-ImageNet}, we use models provided by \cite{torchvision2016}; and for \textsc{CUB-200}, we use a model from \cite{you2019torchcv}.

\paragraph{WideResNet.}
For \textsc{CIFAR-10} and \textsc{CIFAR-100}, we use WideResNet models (WRN-40-6 and WRN-28-6), which are trained from scratch using the code provided in \cite{hendrycks2019augmix}.
Here, WRN-40-6 denotes a WideResNet with depth 40 and widen factor 6, and WRN-28-6 is defined similarly.
For \textsc{CUB-200} and \textsc{Tiny-ImageNet}, we use an ImageNet-pretrained WideResNet-50 model from \cite{torchvision2016} and fine-tune it. 

\paragraph{ViT.}
For \textsc{CIFAR-10}, \textsc{CIFAR-100}, \textsc{CUB-200}, and \textsc{Tiny-ImageNet}, we fine-tune a ViT-Base model with a patch size of $32 \times 32$, initialized from the pretrained model provided at \url{https://huggingface.co/google/vit-base-patch32-224-in21k}.

\subsection{Pseudo-code for ICBM}
\label{app_icbm_algorithm}

Recall that an image model is decomposed as $f(\mathbf{x}) = h(g(\mathbf{x}))$, where $g$ is a feature extractor and $h$ is a linear classifier.
In ICBM, the activation function of the last layer of $g$ is replaced with HTAF during training and with the Heaviside function during inference, i.e.,
\begin{align*}
\text{For training,} \quad g(\mathbf{x}) = \text{HTAF}(q(\mathbf{x})), \qquad
\text{For inference,} \quad g(\mathbf{x}) = \mathbb{I}(q(\mathbf{x}) \geq 0),
\end{align*}
where $q(\mathbf{x})$ is the feature extractor without the activation of the last layer.

The implicit concept of the $j$-th node is interpreted by partitioning the class labels into two groups according to whether the $j$-th node is activated, and providing them to an LLM.
Since the same class can appear in both groups, only images whose HTAF outputs are close to $0$ or $1$ are used.
The procedure is summarized in Algorithm \ref{alg:icbm}, where $g_{j}(\mathbf{x})$ denotes the $j$-th output of the feature extractor with HTAF.

\begin{algorithm}[h]
\caption{Concept interpretation for the $j$-th node}
\label{alg:icbm}
\textbf{Input:} Node index $j$, images $(\mathbf{x}_{1}, \ldots, \mathbf{x}_{n})$, labels $(y_{1}, \ldots, y_{n})$ \\
\textbf{Output:} Implicit concept $c_{j}$
\begin{enumerate}[label=\arabic*:, leftmargin=*, itemsep=0pt]
    \item $\mathcal{G}_{\text{act}} \leftarrow \{ y_{k} : g_{j}(\mathbf{x}_{k}) \approx 1 \}$
    \item $\mathcal{G}_{\text{non}} \leftarrow \{ y_{k} : g_{j}(\mathbf{x}_{k}) \approx 0 \}$
    \item $c_{j} \leftarrow \text{LLM}(\mathcal{G}_{\text{act}}, \mathcal{G}_{\text{non}})$
    \item \textbf{return} $c_{j}$
\end{enumerate}
\end{algorithm}

\newpage


\section{Comparison to ANN-to-SNN method}

\label{app:ann-to-snn}

\begin{table*}[h]
\centering
\footnotesize
\caption{\textbf{Results of the prediction accuracy and inference time comparison between ANN-to-SNN conversion and our method.}}
\vskip -0.2cm
\label{table_ann2snn}
\begin{tabular}{cccccccc}
\toprule
Dataset & Metric & \multicolumn{5}{c}{ANN-to-SNN} & SNN with HTAF \\
\cmidrule(lr){3-7} \cmidrule(lr){8-8}
 &  & $T=10$ & $T=30$ & $T=50$ & $T=100$ & $T=200$ & $T=4$ \\
\midrule
\multirow{2}{*}{\textsc{CIFAR-10}}
& Accuracy 
& 10.0\% & 93.4\% & 37.1\% & 81.0\% & 92.1\%  & 91.0\% \\
& Inference time 
& 8 sec & 27 sec & 46 sec & 85 sec & 183 sec & 3 sec \\
\midrule
\multirow{2}{*}{\textsc{CIFAR-100}}
& Accuracy 
& 1.0\% & 3.5\% & 12.8\% & 47.8\% & 68.9\%  & 66.9\% \\
& Inference time 
& 10 sec &  25 sec & 40 sec & 81 sec & 160 sec & 3 sec \\
\bottomrule
\end{tabular}
\end{table*}

Here, we conduct experiments to compare our method with the ANN-to-SNN conversion method (\cite{rueckauer2017conversion}). 
We implement the ANN-to-SNN code from \cite{doi:10.1126/sciadv.adi1480}.
We use ResNet-18 as the backbone image model for ANN-to-SNN method and hyperparameters are turned as in Appendix \ref{app:detials_all_exp}.
A SNN with HTAF refers to a model that is trained by replacing the Heaviside activation function with HTAF, while at inference time, the activation function is reverted to the Heaviside function.

Table \ref{table_ann2snn} presents the accuracy of the ANN-to-SNN method across different time steps T, along with the accuracy of SNNs with HTAF.
Although ANN-to-SNN achieves higher accuracy at large time steps (e.g., T=200), this comes at a substantial cost in inference time, which increases by approximately 60 times.
In contrast, our method maintains competitive accuracy with significantly fewer time steps (e.g., T=4), resulting in substantially efficiency inference.

\newpage

\section{Additional experiments for Deep Heaviside neural Networks}

\subsection{Prediction performance analysis based on Proposition \ref{thm:heavi_approx}}

\label{app_selct_alpha01}

In this section, we conduct experiments to evaluate the prediction performance of DNNs using HTAF as the activation function under the lower bounds of $\alpha_{0}$ and $\alpha_{1}$ in Proposition \ref{thm:heavi_approx}.
The DNN architecture consists of three hidden layers with 256, 256, and 128 nodes, respectively.
We use the Adam optimizer with a learning rate of 0.001 and train the model for 1,000 epochs.
Experiments are conducted on \textsc{Churn} dataset.

Recall that the lower bounds of $\alpha_{0}$ and $\alpha_{1}$ are given by
\begin{align*}
\alpha_{1} &> \log \left(\frac{1-\epsilon}{\epsilon}\right), \\
\alpha_{0} &\geq \frac{1}{2\epsilon}
\log \left(
\frac{\alpha_{1} + \log\left(\frac{1-\epsilon}{\epsilon}\right)}
{\alpha_{1} - \log\left(\frac{1-\epsilon}{\epsilon}\right)}
\right).
\end{align*}

Specifically, for $\epsilon = 0.1$, the lower bound of $\alpha_{1}$ is 2.20 and the corresponding lower bounds of $\alpha_{0}$ for different values of $\alpha_{1}$ are reported in Table \ref{table_lower_bounds}.
Figure \ref{fig_alpha01_plot} presents the AUROC and binarity score with respect to the lower bound of $\alpha_{0}$ and the value of $\alpha_{1}$ reported in Table \ref{table_lower_bounds}.
The definition of the binarity score is provided in Appendix \ref{app_addexp_bin_concept}.

Theoretically, increasing $\alpha_{1}$ and decreasing $\alpha_{0}$ should improve prediction performance, as this leads to larger gradients in the tail region. 
However, empirical results show that while moderate adjustments can improve prediction performance, excessively large $\alpha_{1}$ or overly small $\alpha_{0}$ induce training instability and degrade prediction performance.
Therefore, careful selection of $\alpha_{0}$ and $\alpha_{1}$ is required.

\begin{table}[h]
\centering
\caption{\textbf{Lower bounds of $\alpha_{0}$ for different values of $\alpha_{1}$.}}
\label{table_lower_bounds}
\begin{tabular}{ccccccccccc}
\toprule
$\alpha_{1}$ value & 5 & 10 & 20 & 30 & 40 & 50 & 100 & 500 & 1000 & 5000 \\ \midrule 
Lower bound of $\alpha_{0}$ & 4.72 & 2.24 & 1.11 & 0.74 & 0.55 & 0.44 & 0.22 & 0.05 & 0.022 & 0.005 \\ \bottomrule
\end{tabular}%
\end{table}

\begin{figure}[h]
    \centering
    \includegraphics[width=1.0\linewidth]{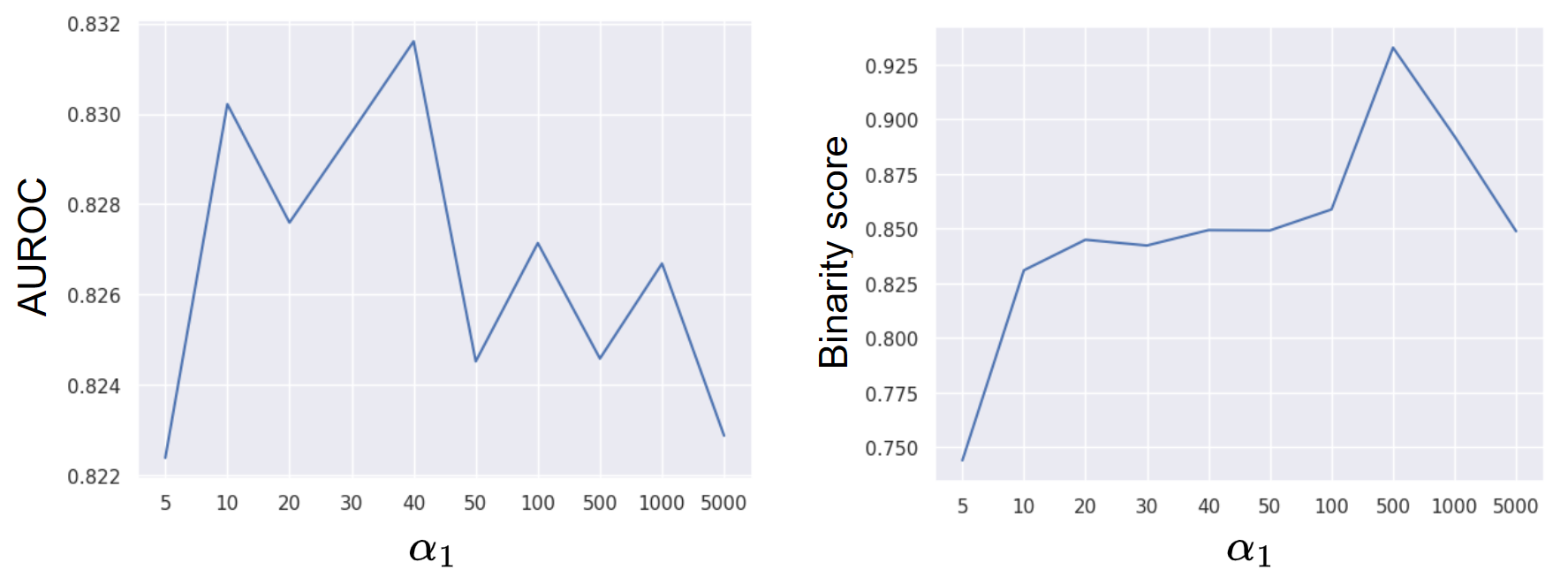}
    \caption{\textbf{Plots of the AUROC and binarity score depends on various $\alpha_{0}$ and $\alpha_{1}$.}}
    \label{fig_alpha01_plot}
\end{figure}

\subsection{Analysis of vanishing gradients}
\label{app:add_grad}

\subsubsection{Mitigating gradient vanishing by reducing $\alpha_{0}$}

\label{app_grad_2}

We aim to investigate whether gradient vanishing can be alleviated by reducing $\alpha_{0}$ while keeping $\alpha_{0}\alpha_{1}$ fixed. 
The DNN architecture is set to [256, 256, 128], and $\beta_{0}$ is fixed at 22. 
We consider three candidates (2, 11), (8, 2.75), and (11, 2) for $(\alpha_{0}, \alpha_{1})$.
Experiments are conducted on \textsc{Abalone} dataset.

We examine the gradients at each layer during training to investigate whether gradient vanishing occurs.
Specifically, following \cite{glorot2010understanding}, we compute the logarithm of the average L2 norm of the gradients of the loss with respect to the pre-activations at each layer over epochs.
We also track the training loss over epochs. 

Figure \ref{fig_gv_plots} presents results for a DNN architecture with hidden layer sizes (256, 256, 128) using HTAF, ReLU, and Scaled Sigmoid.
Here, Batch normalization is not applied in these experiments.

While gradient vanishing is observed for Scaled Sigmoid and HTAF with (8, 2.75) and (11, 2), it is not observed for HTAF with (2, 11) and ReLU.
This suggests that adjusting $\alpha_{0}$ can resolve the vanishing gradient problem.

\begin{figure}[h]
    \centering
    \includegraphics[width=1.0\linewidth]{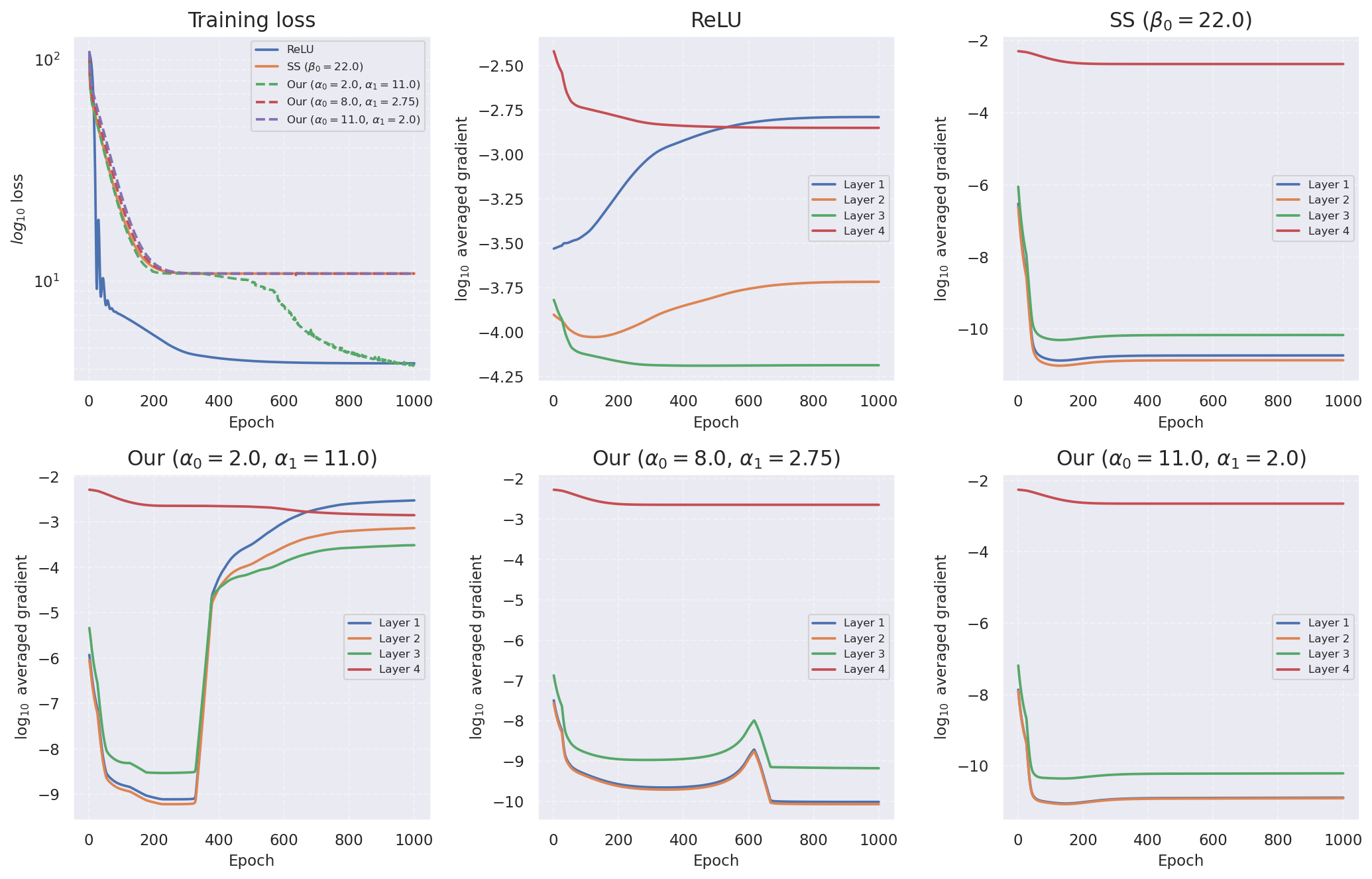}
    \caption{\textbf{Plots of the loss and the logarithm of the average L2 norm of the gradients over epochs for all models.}}
    \label{fig_gv_plots}
\end{figure}

\subsubsection{Analysis of vanishing gradients in various DNN architectures}

\label{app_grad_1}
In this section, we conduct experiments to evaluate the vanishing gradient phenomenon in DNNs when using HTAF, ReLU, and Scaled Sigmoid as activation functions.
We set $\alpha_{0}=3, \alpha_{1}=7.3$ and $\beta_{0}=22$.
We consider various DNN architectures. 
Here, we do not apply batch normalization to each layer.
Experiments are conducted on \textsc{Abalone} dataset.

As described in Appendix \ref{app_grad_2}, we compute the average gradient with respect to the pre-activations at each layer across various DNN architectures.
We investigate the occurrence of gradient vanishing using the average gradients from the first layer closest to the input.
The results are reported in Figure \ref{fig_gv_arch_plots} and Table \ref{table_gradient_vanishing_presence}.
Here, we consider gradient vanishing to occur when the L2 norm of the average gradient value in the first layer is smaller than $10^{-5}$.
Gradient vanishing consistently occurs when using Scaled Sigmoid, whereas it is not observed with ReLU. 
For HTAF, gradient vanishing emerges in networks with five hidden layers.

To mitigate this issue, we decrease $\alpha_{0}$ while keeping the product $\alpha_{0}\alpha_{1}$ constant for DNNs with five hidden layers where gradient vanishing is observed. 
However, this adjustment does not consistently improve prediction performance, nor does it fully resolve the vanishing gradient problem. 

In some cases, as shown in Appendix \ref{app_grad_2}, reducing $\alpha_{0}$ alleviates gradient vanishing and leads to improved performance. 
Overall, these results suggest that tuning $\alpha_{0}$ can partially mitigate gradient vanishing, but does not provide a complete solution.

\begin{figure}[h]
    \centering
    \includegraphics[width=1.0\linewidth]{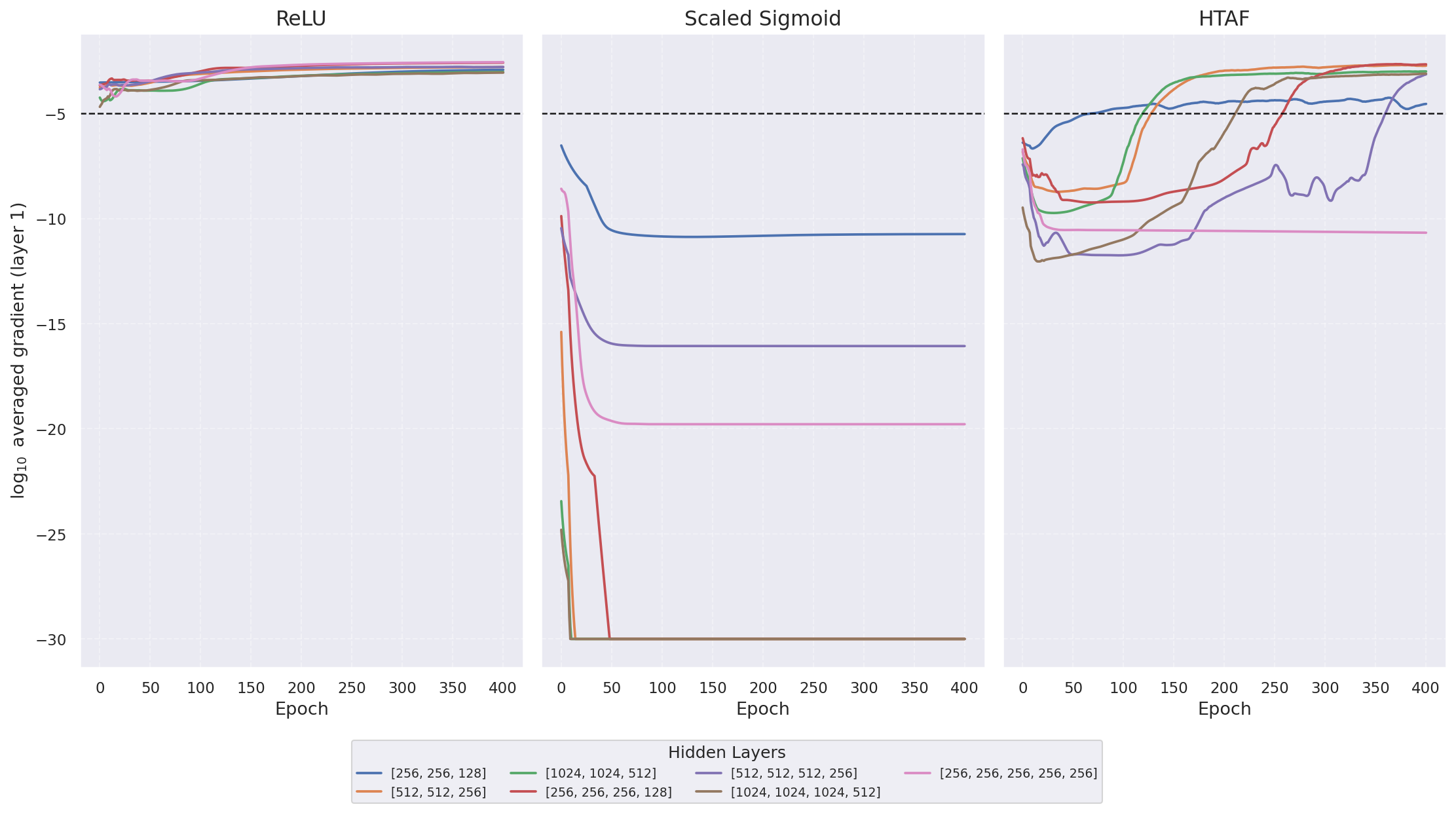}
    \caption{\textbf{Plots of the logarithm of the average L2 norm of layer 1 gradients over epochs for all models.}}
    \label{fig_gv_arch_plots}
\end{figure}

\begin{table}[h]
\centering
\caption{\textbf{Results on the presence of gradient vanishing across different DNN architectures.}}
\label{table_gradient_vanishing_presence}
\begin{tabular}{cccc}
\toprule
Hidden Layer Dimensions of the DNN & ReLU & HTAF & Scaled Sigmoid \\
\midrule
$[256, 256, 128]$ & X & X & O \\
$[512, 512, 256]$ & X & X & O \\
$[1024, 1024, 512]$ & X & X & O \\
$[256, 256, 256, 128]$ & X & X & O \\
$[512, 512, 512, 256]$ & X & X & O \\
$[1024,1024,1024,512]$ & X & X & O \\
$[256,256,256,256,256]$ & X & O & O \\
\bottomrule
\end{tabular}
\end{table}

\subsection{Experiments on Deep Heaviside neural Networks for tabular datasets}

\label{app_tabular_dhn}

In this section, we compare DHNs trained using our HTAF, STE and SS on tabular datasets including \textsc{Wine}, \textsc{Abalone}, \textsc{FICO} and \textsc{Churn}.
The DNN architecture consists of three hidden layers with 256, 256, and 128 nodes.

During training, we replace all activation functions in the DNN with HTAF.
For inference, we use the Heaviside function instead of HTAF.
We refer this model as DNNs with HTAF.
Similarly, we denote the DHN trained by approximating the Heaviside function with Scaled Sigmoid (SS) function as the DNN with SS.
The DHN trained using STE method is referred to as the DHN with STE.

For $\alpha_{0}$, $\alpha_{1}$ and $\beta_{0}$, we choose optimal parameters from the candidate in Appendix \ref{app:details_snn}.
For STE method, we use the candidate surrogate gradients in Appendix \ref{app:details_snn}.

For regression datasets, we use Root Mean Square Error (RMSE) and for classification datasets, we use Area Under ROC curve (AUROC) as the prediction performance measures.
We divide the dataset into training, validation, and test sets with a 70/10/20 ratio. 
The validation set is used to determine the optimal hyperparameters, and the test set is used to evaluate the prediction performance of the trained models.
We repeat this random splitting procedure five times to obtain five measures for prediction performance.

Table \ref{table_DHN_Result} presents that the averages and standard errors of the performance measurement of the DNN, the DNN with HTAF, the DNN with SS and the DHN with the STE on \textsc{Wine}, \textsc{Abalone}, \textsc{Churn} and \textsc{FICO} datasets.
The results indicate that the DNN with HTAF achieves markedly better prediction performance than the DNN with SS and the DHN with STE method.

\begin{table}[h]
    \centering
    \footnotesize
    \caption{\textbf{Results of prediction performance on various tabular datasets.}}
    \label{table_DHN_Result}
    \begin{tabular}{cccccc}
        \toprule
        Dataset & Measure & DNN & DNN + HTAF & DNN + SS & DHN + STE \\
        \midrule
        \textsc{Wine} & RMSE $\downarrow$ & 0.693 (0.01) & 0.730 (0.01) & 0.745 (0.01) & 0.794 (0.01)\\
        \textsc{Abalone} & RMSE $\downarrow$ & 2.003 (0.02) &   2.106 (0.02) & 2.403 (0.08) & 2.620 (0.07)\\
        \textsc{Churn} & AUROC $\uparrow$ & 0.831 (0.01) & 0.828 (0.01)  & 0.828 (0.02) & 0.816 (0.01)\\
        \textsc{FICO} & AUROC $\uparrow$ & 0.787 (0.01) & 0.778 (0.01) & 0.778 (0.01) &0.773 (0.01)\\
        \bottomrule
    \end{tabular}
\end{table}

\subsection{Experiments on Deep Heaviside neural Networks for Large Language Models}
In this section, we evaluate the effectiveness of replacing activation functions in LLM with HTAF on commonsense reasoning tasks.
We consider two approaches. 
First, we replace the activation function with HTAF during training and revert it to the Heaviside activation function at inference time.
We refer this approach as HTAF method.
Second, we replace the activation function with the Heaviside function and train the model using STE.
We refer this approach as STE method.

We use the pre-trained Llama-2-7b (\cite{touvron2023llama}) and Qwen-3.5-9b (\cite{qwen3.5}) models as the backbone architectures, initialized from the pretrained checkpoints provided at \url{https://huggingface.co/meta-llama/Llama-2-7b-hf} and \url{https://huggingface.co/Qwen/Qwen3.5-9B}, respectively.
The decoders of both Llama-2-7b and Qwen-3.5-9b consist of 32 layers, each of which contains a feedforward module with the SwiGLU (\cite{shazeer2020glu}) activation function.
To analyze the impact of HTAF in deep layers, we progressively replace the activation functions in the last $k$ layers of the decoder, where $k$ takes values in $\{4, 8, 12, 16, 20, 24, 28, 32\}$.
We consider three common-sense reasoning benchmarks : Winogrande-small (WG-S, \cite{sakaguchi2021winogrande}), Winogrande-medium (WG-M, \cite{sakaguchi2021winogrande}), OpenBookQA (OBQA, \cite{mihaylov2018can}).

We apply both HTAF method and STE method to the pre-trained Llama-2-7b and the pre-trained Qwen-3.5-9b, and fine-tune the model on each dataset using Low-Rank Adaptation (LoRA, \cite{hu2022lora}). 
For training, we use the AdamW optimizer with a linear learning rate scheduler, a warmup ratio of 0.06, and a learning rate of \(10^{-4}\). 
For hyperameters of LoRA, we use rank \(r = 8\) and scaling factor \(\alpha = 16\), and each stage is trained for 5,000 steps.

Figure \ref{fig_LLM_result} presents the accuracies of Llama-2-7b and Qwen-3.5-9b as a function of $k$.
As $k$ increases, the overall prediction performance of the model degrades; however, the HTAF method exhibits a smaller performance drop compared to the STE method. 
These results suggest that HTAF provides a more stable alternative when replacing activation functions with the Heaviside function in LLMs.

\begin{figure}[h]
    \centering
    \includegraphics[width=1.0\linewidth]{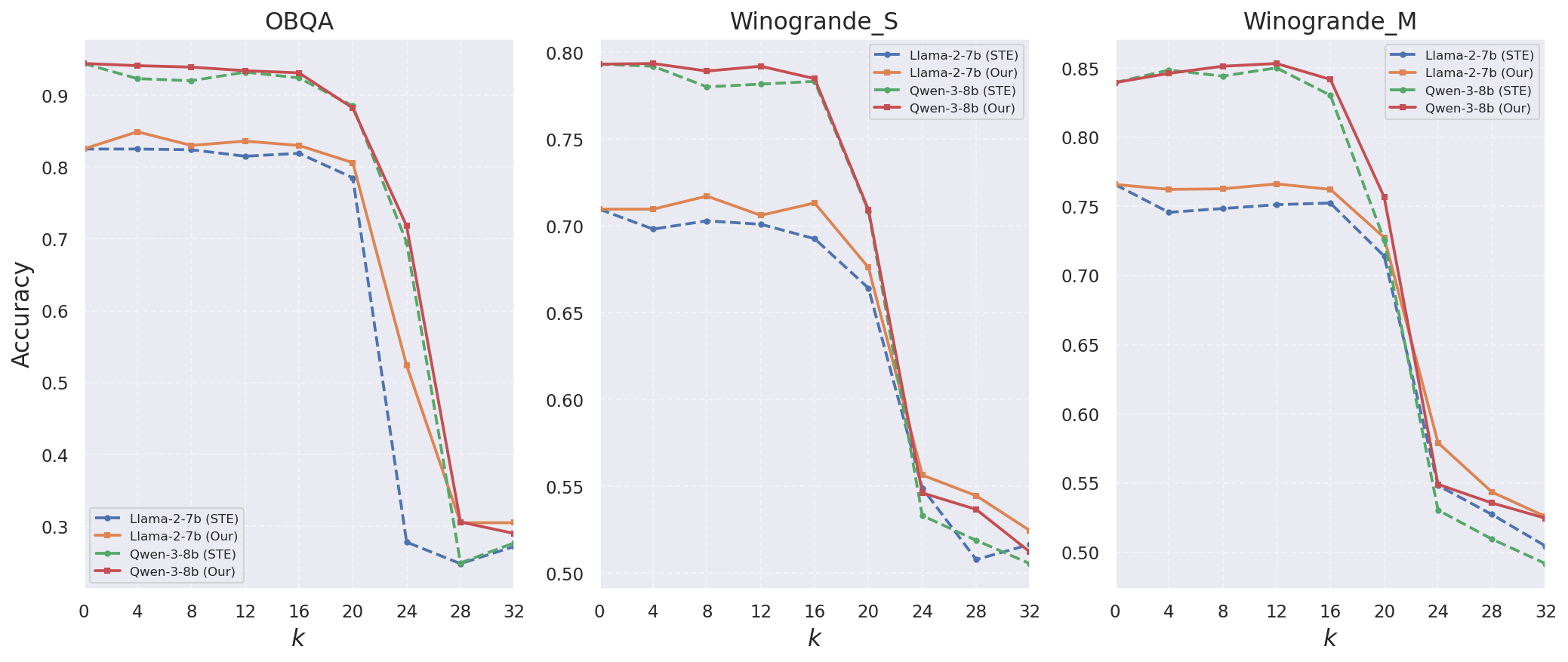}
    \caption{\textbf{Results of the prediction accuracy of LLM depends on various $k$.}}
    \label{fig_LLM_result}
\end{figure}

\newpage

\section{Additional Experiments for ICBM}
\label{app_ICBM_results}

\subsection{Additional results in Section \ref{sec:interpret}}
\label{app_LLM_responce}

Here, we provide the LLM responses for each node described in Section \ref{sec:interpret}.






\begin{figure}[h]
    \centering
    \includegraphics[width=0.9\linewidth]{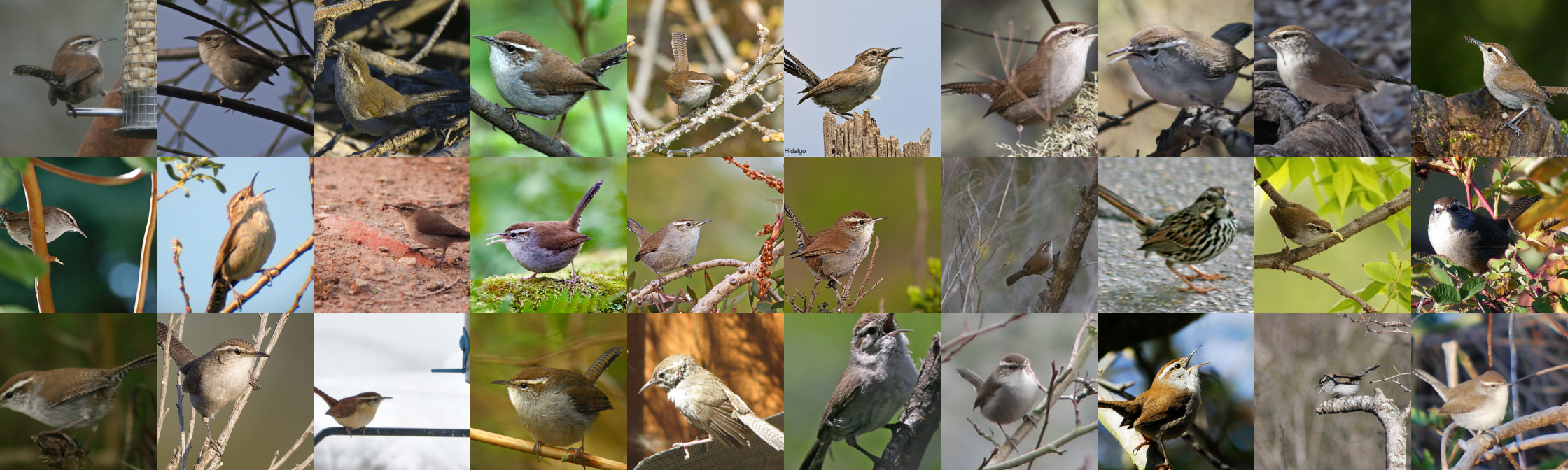}
    \caption{Images of \textit{Bewick Wren}.}
    \label{fig:bewick_wren}
\end{figure}

\paragraph{135th node.}
\textit{Bewick's Wren} is a terrestrial perching bird that typically forages in brush, shrubs, woodland edges, bark crevices, and dense vegetation, searching for small insects and other invertebrates. Hence, they are likely to be activated.

\begin{tcolorbox}[colback=gray!10, colframe=gray, title=LLM Response]
\begin{footnotesize}
\begin{verbatim}
Distinguishing Concept
Ecomorphological Niche: Terrestrial Woodland Foragers vs. Aquatic, Marine, 
                        and Aerial Specialists

The primary concept that separates these two groups is their ecological niche, 
which dictates their habitat, morphological adaptations (bill shape, foot 
structure), and foraging strategies. Group A consists entirely of land-dwelling 
perching birds adapted to forests and brush, while Group B is heavily dominated 
by waterbirds, seabirds, and birds adapted to high-speed aerial hunting.
\end{verbatim}
\end{footnotesize}
\end{tcolorbox}

\paragraph{186th node.}
Since \textit{Bewick's Wren} has slender, pointed, slightly decurved bill suited for precisely capturing small insects and invertebrates, it is likely to be activated in the node.
\begin{tcolorbox}[colback=gray!10, colframe=gray, title=LLM Response]
\begin{footnotesize}
\begin{verbatim}
Distinguishing Concept
Bill Morphology: Hooked/Spear-like vs. Stout or Conical

The primary concept separating these two groups is the functional morphology 
of the bill. Group A is characterized by bills specialized for capturing mobile
prey, including wide gapes for aerial insect capture, hooked bills for grasping 
or tearing prey, and spear-like bills for plunge-diving. In contrast, Group B is 
characterized by bills adapted for applying mechanical force, such as stout 
conical bills for cracking seeds and sturdy chisel-like bills for excavating wood.
\end{verbatim}
\end{footnotesize}
\end{tcolorbox}

\vspace{2mm}
\paragraph{690th node.}
\textit{Bewick's Wren} belongs to unactivated group because it is a cryptic, substrate-bound thicket dweller: it lives in dense woody cover, probes bark and leaf litter for hidden insects.
\begin{tcolorbox}[colback=gray!10, colframe=gray, title=LLM Response]
\begin{footnotesize}
\begin{verbatim}
Distinguishing Concept
Ecological Biome Dependency: Grassland/Prairie Avifauna vs. Thicket, Shrubland,
                             and Wetland Avifauna

The defining concept here is the macro-habitat and vegetative density these 
birds require to survive and breed. Group A represents an ecological guild 
deeply tied to open, low-growth environments (prairies, agricultural fields, 
and tundra). Group B represents a guild dependent on dense, woody secondary 
growth (briars, thickets, and scrub) or aquatic/riparian zones.
\end{verbatim}
\end{footnotesize}
\end{tcolorbox}

\vspace{2mm}
\paragraph{153rd node.}
The \textit{Bewick's Wren} firmly belongs to the unactivated group because its preference for dense, brushy habitats aligns with that group's thicket-dwelling ecological profile.
\begin{tcolorbox}[colback=gray!10, colframe=gray, title=LLM Response]
\begin{footnotesize}
\begin{verbatim}
Distinguishing Concept
Bill Ecomorphology & Foraging Mechanics: Precision Snatching & Gleaning vs. 
                                         Mechanical Extraction & Crushing

The defining concept that separates these two groups is the structural engineering
of their bills and the physical force they apply to secure food. 
Group A is characterized by "light-duty," highly precise bills designed to 
snatch exposed, unarmored food (like flying insects or swimming fish) using speed 
and agility. Group B is characterized by "heavy-duty" or highly specialized bills 
designed to apply significant mechanical force to break into armored food sources 
(wood, hard seeds, pine cones) or probe deeply into tight crevices.
\end{verbatim}
\end{footnotesize}
\end{tcolorbox}

\paragraph{217th node.}
As can be seen in Figure \ref{fig:bewick_wren}, \textit{Bewick's Wren} lacks the bright yellow carotenoids that define the unactivated group, it is more likely to be activated in the node.
\begin{tcolorbox}[colback=gray!10, colframe=gray, title=LLM Response]
\begin{footnotesize}
\begin{verbatim}
Distinguishing Concept
Visual Plumage Pigmentation: Yellow/Carotenoid Dominance vs. Structural, Melanin, 
                             and Red Coloration

While there are ecological differences between the two groups, the most defining
and unifying factor—particularly in datasets categorizing these specific bird 
species (such as the CUB-200 ornithology dataset)—is their visual appearance, 
specifically the dominant color of their plumage. Group B is overwhelmingly 
characterized by bright yellow plumage (driven by dietary carotenoids), whereas 
Group A is defined by the absence of yellow, featuring structural blues, deep 
reds, and heavily melanistic grays, blacks, and browns.
\end{verbatim}
\end{footnotesize}
\end{tcolorbox}

\vspace{1mm}
\textit{Bewick Wren} exhibits neither `has\_primary\_color::black' nor `has\_wing\_color::orange', and is therefore activated at this node.

\vspace{2mm}
\subsection{Definition of importance score for the estimated implicit concept}

\label{app_def_score}

Let $\mathbf{c} = (c_{1},...,c_{N}) \in \{0,1\}^{N}$ be a implicit concept vector (i.e., the outputs of the feature extractor), where $N$ is the number of nodes in the final layer of the feature extractor.
Since the classifier is a linear model, for a given class $k$, the logit value is defined as
\begin{align}
\text{logit}(k) = \sum_{j=1}^{N}w_{j,k}c_{j} + b_{k},
\label{eq:boundary_k}
\end{align}
where $w_{j,k}$'s are the weights and $b_{k}$ is the bias of the linear model.

That is, the importance of each estimated implicit concept is quantified via Linear model's weights.
Specifically, we define the importance score of the estimated implicit concept of the $j$th node for class $k$ as the absolute value of weight, i.e., $|w_{j,k}|$. 
Higher importance score indicates that the corresponding node contributes more significantly to the classification of the given class.

\subsection{Additional results for class-specific activation values}
\label{app_class_specific}

In this section, we provide additional results for all classes corresponding to the experiment in Section \ref{sec:class-spec} on CIFAR-10 dataset. 
Figure \ref{fig_hist_all} presents these results.
It implies that the implicit concepts estimated by ICBM are highly class-specific, whereas those learned by a standard image model are not.

\begin{footnotesize}
\begin{figure}[h]
    \centering
    \includegraphics[width=1\linewidth]{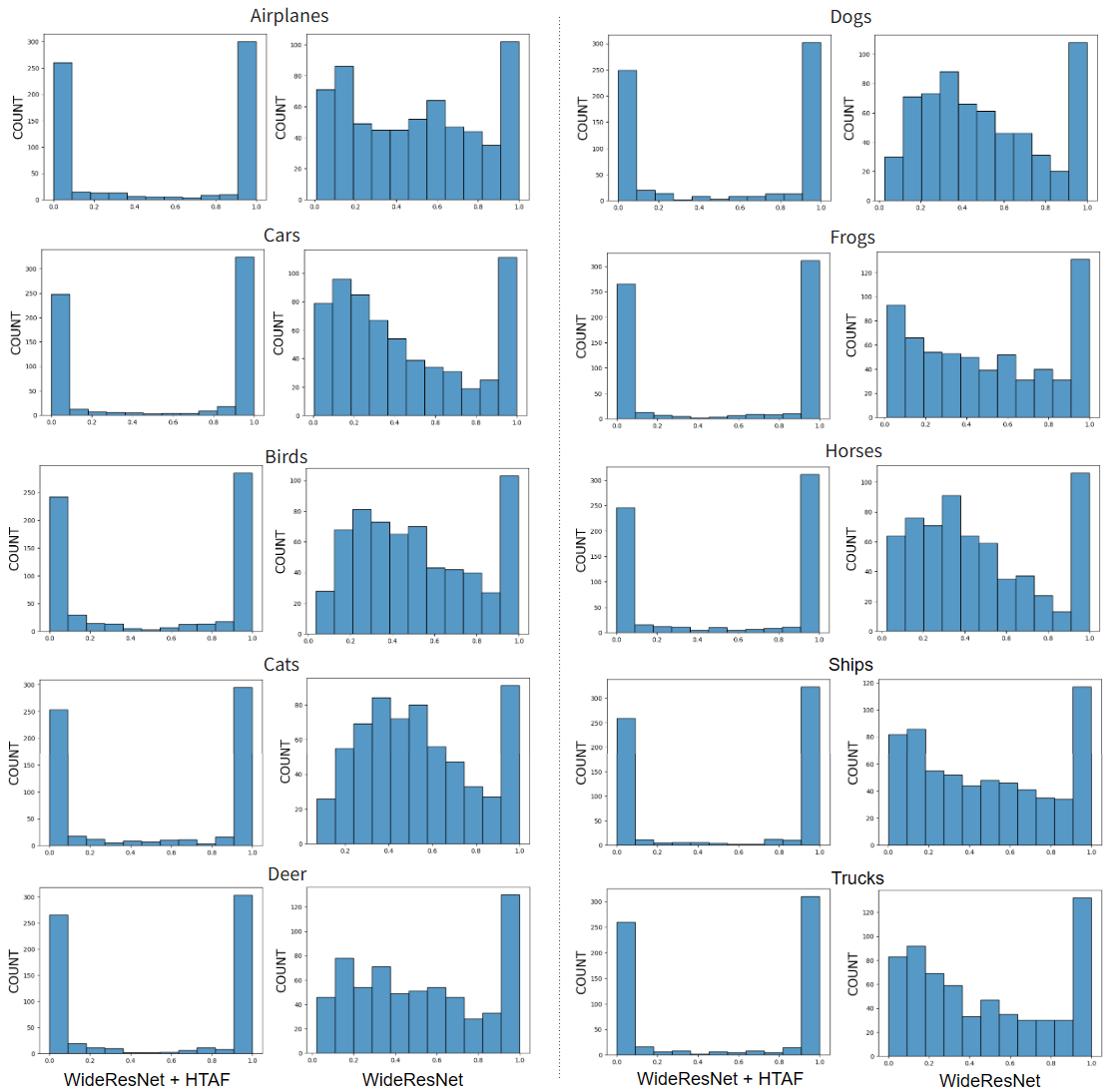}
    \caption{\textbf{Histograms of the proportions for all classes.}}
    \label{fig_hist_all}
\end{figure}
\end{footnotesize}

\subsection{Additional experiments on Heaviside approximation and concept meaningfulness}

\label{app_addexp_bin_concept}

\begin{table*}[h]
\centering
\scriptsize
\caption{\textbf{Comparison between HTAF, SS, STE and BM.}}
\vskip -0.2cm
\label{table_scaledsigmoid}
\begin{tabular}{ccccccccc}
\toprule
Dataset             & Measure  & HTAF 
& \hskip -0.3cm \begin{tabular}{c}
SS \\
($\beta_{0} = 10^{1}$)\\
\end{tabular} & \hskip -0.3cm \begin{tabular}{c}
SS \\
($\beta_{0} = 50$)\\
\end{tabular} & \hskip -0.3cm \begin{tabular}{c}
SS \\
($\beta_{0} = 10^{2}$)\\
\end{tabular} & \hskip -0.3cm \begin{tabular}{c}
SS \\
($\beta_{0} = 10^{3}$)\\
\end{tabular} & \hskip -0.3cm \begin{tabular}{c}
STE\\
\end{tabular} & \hskip -0.3cm \begin{tabular}{c}
BM \\
\end{tabular} 
\\ \midrule
\multirow{3}{*}{\textsc{CIFAR-10}} & Accuracy  & 96.1\% & 95.5\% & 95.7\% & 95.3\% & 79.3\% & 94.3\% & 90.0\%\\  
           & Binarity score & 98.4\% &  3.9\% & 40.4\% & 73.3 \% & 99.7 \% & 100.0\% & --\\ 
           & Concept accuracy & 96.1\% &  91.2 \% & 91.6\% &91.5\% & 73.8 \% & 89.7\% & 88.4\% \\ \midrule
\multirow{3}{*}{\textsc{CIFAR-100}} & Accuracy & 79.5\% & 77.8\% & 78.6\% & 78.2 \% & 56.3\% & 74.8\% & 70.7\%\\  
                  & Binarity score & 98.3\% & 14.3 \% & 60.5\% & 74.0 \% & 99.8\% & 100.0\% &  --\\
                  & Concept accuracy & 84.6\% & 77.0\% & 78.7\% & 78.8\% & 67.8\%  & 76.1\% & 70.8\%\\
                  \bottomrule
\end{tabular}
\end{table*}

In this section, we evaluate whether our method approximates the Heaviside activation function while achieving strong prediction performance. 
We also investigate whether the estimated implicit concepts contribute meaningfully to the classification decision.

To measure how close the activation outputs are to zero or one, we use binarity score, defined as \textit{the proportion of activation outputs that are greater than 0.99 or less than 0.01}.
Higher binarity score indicates that the activation is closer to the binary representations.

Furthermore, we evaluate whether the implicit concepts estimated by ICBMs capture meaningful information for image classification under the following assumption: if the presence or absence of an estimated implicit concept can be predicted from the true class labels of images, then the concept is likely to encode semantically meaningful information for classification.
Therefore, for each node, we train a logistic regression model to predict the absence of the estimated implicit concept using only the true class label as input. 
We measure the meaningfulness of the implicit concepts estimated by the model by averaging the accuracies across all nodes, and refer to this metric as a concept accuracy. 
A higher concept accuracy indicates that the model estimates more meaningful implicit concepts.

We use WideResNet as the backbone model and conduct experiments on the CIFAR-10 and CIFAR-100 datasets.
We consider four methods, including ours, as follows:
\begin{itemize}
    \item HTAF : The activation function in the final layer of the feature extractor is replaced with HTAF during training, and the outputs are binarized at inference time.
    \item SS : The activation function in the final layer of the feature extractor is replaced with Scaled Sigmoid during training, and the outputs are binarized at inference time.
    \item BM : At inference time, the outputs of the feature extractor are binarized based on a median threshold.
    \item STE : Replaces the activation with the Heaviside function and trains it using Straight-Through Estimator (STE).
\end{itemize}

For all methods, hyperparameters are chosen as described in Appendix \ref{app:detials_all_exp}.
For BM, the training-related hyperparameters are tuned in the same manner as in the other methods.
Table \ref{table_scaledsigmoid} presents the averages of accuracies, binarity scores and concept accuracies for all methods.

We observe that HTAF achieves the best overall performance across all metrics. 
For SS, binarity score increases as $\beta_{0}$ becomes larger; however, training becomes unstable, leading to degraded prediction performance.
In the case of BN, replacing the activation function with the Heaviside activation function at inference time substantially degrades prediction performance.

\subsection{Effect of post-binarization on model prediction performance}

We conduct an additional experiment to assess the impact of inference-time binarization, defined as replacing the activation function with the exact Heaviside function during inference. 
Specifically, we measure how much the prediction performance degrades after this replacement.
We set the WideResNet as the backbone image model and consider CIFAR-10 and CIFAR-100 datasets.
Table \ref{table:binarization_result} presents the average accuracy before and after inference-time binarization for HTAF, SS, and BM defined in Appendix \ref{app_addexp_bin_concept}, over five trials.
We observe that HTAF exhibits the smallest degradation in prediction performance after binarization, indicating that it is more robust to replacing the activation function with the exact Heaviside function at inference time.

\begin{table}[h]
    \centering
    \caption{\textbf{Averaged accuracy before and after inference-time binarization.}}
    \label{table:binarization_result}
    \begin{tabular}{ccccccc}
        \toprule
        \multirow{2}{*}{Dataset} 
        & \multicolumn{2}{c}{HTAF} 
        & \multicolumn{2}{c}{SS} 
        & \multicolumn{2}{c}{BM} \\
        \cmidrule(lr){2-3} \cmidrule(lr){4-5} \cmidrule(lr){6-7}
        & Before & After & Before & After & Before & After \\
        \midrule
        \textsc{CIFAR-10}  
        & 96.1\% & 96.1\% 
        & 96.1\% & 95.7\% 
        & 95.6\% & 90.0\% \\
        \textsc{CIFAR-100} 
        & 79.5\% & 79.5\% 
        & 79.0\% & 78.7\% 
        & 79.2\% & 70.7\% \\
        \bottomrule
    \end{tabular}
\end{table}

\subsection{Analysis of representations by HTAF}

In this section, we analyze the outputs of the feature extractor in ICBMs when using HTAF compared to ReLU.
We use the CIFAR-100 dataset, and adopt WideResNet as the backbone image model.
Figure \ref{fig:act_output} presents the histogram of feature extractor outputs aggregated over all feature dimensions and all samples in the test dataset.
We observe that, when using HTAF, most of the feature extractor outputs are concentrated near zero or one, whereas this behavior is not observed when using ReLU.

\begin{figure}[h]
    \centering
    \includegraphics[width=1.0\linewidth]{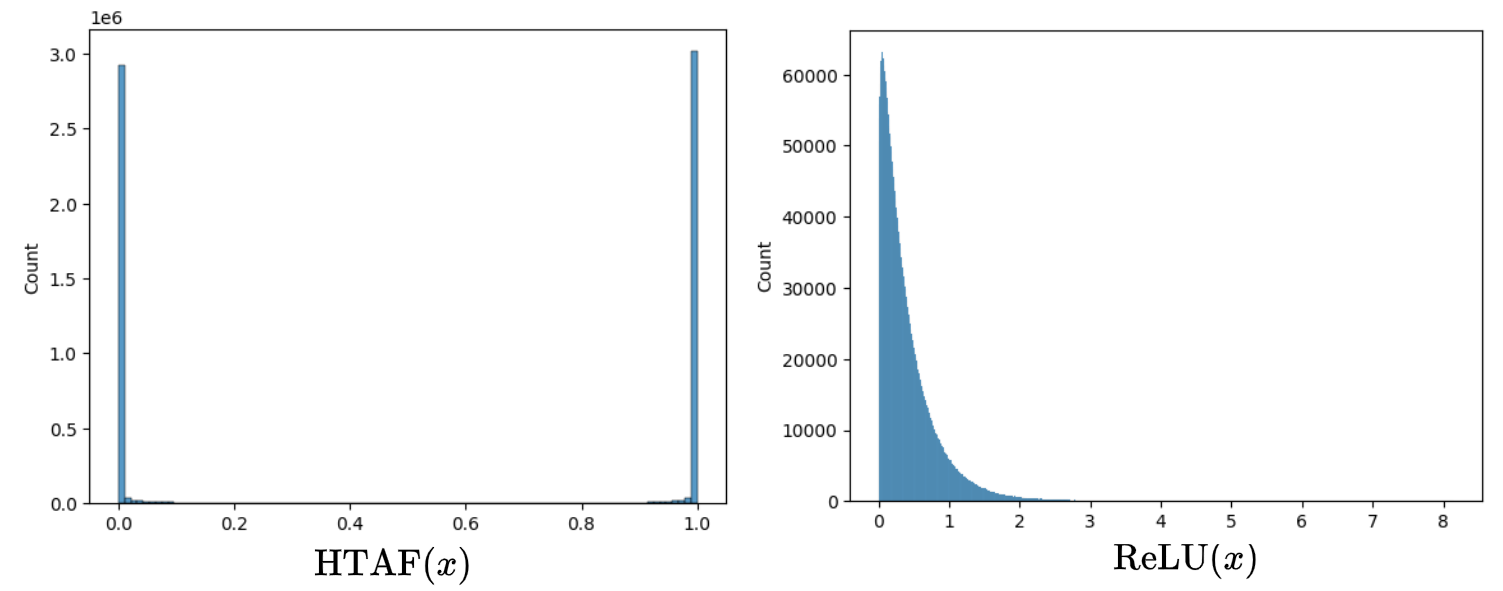}
    \caption{\textbf{Plots of the outputs for HTAF and ReLU.}}
    \label{fig:act_output}
\end{figure}




\subsection{Comparision with CBMs}
\label{app:compare_CBM}

Here, we conduct additional experiment to compare ICBMs with original CBMs (\cite{koh2020concept}) and Label-free CBMs (\cite{oikarinen2023label}) on CIFAR-10, CIFAR-100 and CUB-200 datasets.
The backbone image model is set to WideResNet. 
For Label-free CBMs, we use the official code provided by \cite{oikarinen2023label}, with minor modifications.
For CBMs, we implement the model from scratch.
We employ SGD as the optimizer with a learning rate of $10^{-4}$. 
All other hyperparameters for Label-free CBMs are set to the default values in the official code.
Table \ref{table:result_CBM} presents the prediction performance of ICBMs, CBMs, and Label-free CBMs. 
The best results are highlighted in \textbf{bold}.
For CIFAR-10 and CIFAR-100 datasets, original CBMs cannot be trained due to the absence of concept annotations. 
As shown in the table, ICBMs achieve substantially higher prediction performance compared to baseline methods.

\begin{table}[h]
\centering
\caption{\textbf{Results of accuracies of CBMs, Label-free CBMs and ICBMs on \textsc{CIFAR-10}, \textsc{CIFAR-100} and \textsc{CUB-200} datasets.}}
\label{table:result_CBM}
\begin{tabular}{cccc}
\toprule
Model & CIFAR-10 & CIFAR-100 & CUB-200 \\ \midrule
CBMs & -- & -- & 72.5\% \\
Label-free CBMs & 86.5\% & 65.1\% & 55.7\% \\
ICBMs (Ours) & \textbf{96.1}\% & \textbf{79.5}\% & \textbf{77.3}\% \\
\bottomrule
\end{tabular}%
\end{table}

\newpage

\subsection{Experiments of Computational cost}

We conduct an additional experiment to compare the computational costs of HTAF with those of STE and Scaled Sigmoid. We trained a deep neural network consisting of three hidden layers with 256, 256, and 128 neurons, respectively, on the Abalone dataset.
The training time for 100 epochs is measured and the experiment is repeated 10 times. We report the average training time and its standard error in Table \ref{table:training_time}.
As shown in Table 1, HTAF is slightly slower than STE and Scaled Sigmoid.
However, the difference is marginal, even when considering the total training time for 100 epochs. 

\begin{table}[h]
\centering
\footnotesize
\caption{\textbf{Comparison of the average training time (100 epochs) among HTAF, STE, and Scaled Sigmoid (SS) on \textsc{Abalone} dataset.}}
\label{table:training_time}
\begin{tabular}{cccc}
\toprule
 & HTAF & STE & SS \\
\midrule
Average time & 1.98 sec (0.01) & 1.91 sec (0.02) & 1.90 sec (0.03) \\
\bottomrule
\end{tabular}
\end{table}

\section*{NeurIPS Paper Checklist}

The checklist is designed to encourage best practices for responsible machine learning research, addressing issues of reproducibility, transparency, research ethics, and societal impact. Do not remove the checklist: {\bf The papers not including the checklist will be desk rejected.} The checklist should follow the references and follow the (optional) supplemental material.  The checklist does NOT count towards the page
limit. 

Please read the checklist guidelines carefully for information on how to answer these questions. For each question in the checklist:
\begin{itemize}
    \item You should answer \answerYes{}, \answerNo{}, or \answerNA{}.
    \item \answerNA{} means either that the question is Not Applicable for that particular paper or the relevant information is Not Available.
    \item Please provide a short (1--2 sentence) justification right after your answer (even for \answerNA). 
\end{itemize}

{\bf The checklist answers are an integral part of your paper submission.} They are visible to the reviewers, area chairs, senior area chairs, and ethics reviewers. You will also be asked to include it (after eventual revisions) with the final version of your paper, and its final version will be published with the paper.

The reviewers of your paper will be asked to use the checklist as one of the factors in their evaluation. While \answerYes{} is generally preferable to \answerNo{}, it is perfectly acceptable to answer \answerNo{} provided a proper justification is given (e.g., error bars are not reported because it would be too computationally expensive'' or ``we were unable to find the license for the dataset we used''). In general, answering \answerNo{} or \answerNA{} is not grounds for rejection. While the questions are phrased in a binary way, we acknowledge that the true answer is often more nuanced, so please just use your best judgment and write a justification to elaborate. All supporting evidence can appear either in the main paper or the supplemental material, provided in appendix. If you answer \answerYes{} to a question, in the justification please point to the section(s) where related material for the question can be found.

IMPORTANT, please:
\begin{itemize}
    \item {\bf Delete this instruction block, but keep the section heading ``NeurIPS Paper Checklist"},
    \item  {\bf Keep the checklist subsection headings, questions/answers and guidelines below.}
    \item {\bf Do not modify the questions and only use the provided macros for your answers}.
\end{itemize}


\begin{enumerate}

\item {\bf Claims}
    \item[] Question: Do the main claims made in the abstract and introduction accurately reflect the paper's contributions and scope?
    \item[] Answer: \answerYes{} 
    \item[] Justification: The abstract and introduction accurately summarize the paper’s contributions and scope. The main ideas and findings are consistently supported by the experiments and analysis.
    \item[] Guidelines:
    \begin{itemize}
        \item The answer \answerNA{} means that the abstract and introduction do not include the claims made in the paper.
        \item The abstract and/or introduction should clearly state the claims made, including the contributions made in the paper and important assumptions and limitations. A \answerNo{} or \answerNA{} answer to this question will not be perceived well by the reviewers. 
        \item The claims made should match theoretical and experimental results, and reflect how much the results can be expected to generalize to other settings. 
        \item It is fine to include aspirational goals as motivation as long as it is clear that these goals are not attained by the paper. 
    \end{itemize}

\item {\bf Limitations}
    \item[] Question: Does the paper discuss the limitations of the work performed by the authors?
    \item[] Answer: \answerYes{} 
    \item[] Justification: We discuss the limitations in the conclusion section.
    \item[] Guidelines:
    \begin{itemize}
        \item The answer \answerNA{} means that the paper has no limitation while the answer \answerNo{} means that the paper has limitations, but those are not discussed in the paper. 
        \item The authors are encouraged to create a separate ``Limitations'' section in their paper.
        \item The paper should point out any strong assumptions and how robust the results are to violations of these assumptions (e.g., independence assumptions, noiseless settings, model well-specification, asymptotic approximations only holding locally). The authors should reflect on how these assumptions might be violated in practice and what the implications would be.
        \item The authors should reflect on the scope of the claims made, e.g., if the approach was only tested on a few datasets or with a few runs. In general, empirical results often depend on implicit assumptions, which should be articulated.
        \item The authors should reflect on the factors that influence the performance of the approach. For example, a facial recognition algorithm may perform poorly when image resolution is low or images are taken in low lighting. Or a speech-to-text system might not be used reliably to provide closed captions for online lectures because it fails to handle technical jargon.
        \item The authors should discuss the computational efficiency of the proposed algorithms and how they scale with dataset size.
        \item If applicable, the authors should discuss possible limitations of their approach to address problems of privacy and fairness.
        \item While the authors might fear that complete honesty about limitations might be used by reviewers as grounds for rejection, a worse outcome might be that reviewers discover limitations that aren't acknowledged in the paper. The authors should use their best judgment and recognize that individual actions in favor of transparency play an important role in developing norms that preserve the integrity of the community. Reviewers will be specifically instructed to not penalize honesty concerning limitations.
    \end{itemize}

\item {\bf Theory assumptions and proofs}
    \item[] Question: For each theoretical result, does the paper provide the full set of assumptions and a complete (and correct) proof?
    \item[] Answer: \answerYes{} 
    \item[] Justification: We provide the proofs in Appendix \ref{app_proofs}, Appendix \ref{app:minmax}, and Appendix \ref{app_dhn_approx}.
    \item[] Guidelines:
    \begin{itemize}
        \item The answer \answerNA{} means that the paper does not include theoretical results. 
        \item All the theorems, formulas, and proofs in the paper should be numbered and cross-referenced.
        \item All assumptions should be clearly stated or referenced in the statement of any theorems.
        \item The proofs can either appear in the main paper or the supplemental material, but if they appear in the supplemental material, the authors are encouraged to provide a short proof sketch to provide intuition. 
        \item Inversely, any informal proof provided in the core of the paper should be complemented by formal proofs provided in appendix or supplemental material.
        \item Theorems and Lemmas that the proof relies upon should be properly referenced. 
    \end{itemize}

    \item {\bf Experimental result reproducibility}
    \item[] Question: Does the paper fully disclose all the information needed to reproduce the main experimental results of the paper to the extent that it affects the main claims and/or conclusions of the paper (regardless of whether the code and data are provided or not)?
    \item[] Answer: \answerYes{} 
    \item[] Justification: We provide information on the hyperparameters, model architectures, and datasets in Appendix \ref{app:detials_all_exp}.
    \item[] Guidelines:
    \begin{itemize}
        \item The answer \answerNA{} means that the paper does not include experiments.
        \item If the paper includes experiments, a \answerNo{} answer to this question will not be perceived well by the reviewers: Making the paper reproducible is important, regardless of whether the code and data are provided or not.
        \item If the contribution is a dataset and\slash or model, the authors should describe the steps taken to make their results reproducible or verifiable. 
        \item Depending on the contribution, reproducibility can be accomplished in various ways. For example, if the contribution is a novel architecture, describing the architecture fully might suffice, or if the contribution is a specific model and empirical evaluation, it may be necessary to either make it possible for others to replicate the model with the same dataset, or provide access to the model. In general. releasing code and data is often one good way to accomplish this, but reproducibility can also be provided via detailed instructions for how to replicate the results, access to a hosted model (e.g., in the case of a large language model), releasing of a model checkpoint, or other means that are appropriate to the research performed.
        \item While NeurIPS does not require releasing code, the conference does require all submissions to provide some reasonable avenue for reproducibility, which may depend on the nature of the contribution. For example
        \begin{enumerate}
            \item If the contribution is primarily a new algorithm, the paper should make it clear how to reproduce that algorithm.
            \item If the contribution is primarily a new model architecture, the paper should describe the architecture clearly and fully.
            \item If the contribution is a new model (e.g., a large language model), then there should either be a way to access this model for reproducing the results or a way to reproduce the model (e.g., with an open-source dataset or instructions for how to construct the dataset).
            \item We recognize that reproducibility may be tricky in some cases, in which case authors are welcome to describe the particular way they provide for reproducibility. In the case of closed-source models, it may be that access to the model is limited in some way (e.g., to registered users), but it should be possible for other researchers to have some path to reproducing or verifying the results.
        \end{enumerate}
    \end{itemize}

\item {\bf Open access to data and code}
    \item[] Question: Does the paper provide open access to the data and code, with sufficient instructions to faithfully reproduce the main experimental results, as described in supplemental material?
    \item[] Answer: \answerYes{} 
    \item[] Justification: We will release the code on GitHub upon acceptance of the paper, and the implementation source code has been included in the supplementary material.
    \item[] Guidelines:
    \begin{itemize}
        \item The answer \answerNA{} means that paper does not include experiments requiring code.
        \item Please see the NeurIPS code and data submission guidelines (\url{https://neurips.cc/public/guides/CodeSubmissionPolicy}) for more details.
        \item While we encourage the release of code and data, we understand that this might not be possible, so \answerNo{} is an acceptable answer. Papers cannot be rejected simply for not including code, unless this is central to the contribution (e.g., for a new open-source benchmark).
        \item The instructions should contain the exact command and environment needed to run to reproduce the results. See the NeurIPS code and data submission guidelines (\url{https://neurips.cc/public/guides/CodeSubmissionPolicy}) for more details.
        \item The authors should provide instructions on data access and preparation, including how to access the raw data, preprocessed data, intermediate data, and generated data, etc.
        \item The authors should provide scripts to reproduce all experimental results for the new proposed method and baselines. If only a subset of experiments are reproducible, they should state which ones are omitted from the script and why.
        \item At submission time, to preserve anonymity, the authors should release anonymized versions (if applicable).
        \item Providing as much information as possible in supplemental material (appended to the paper) is recommended, but including URLs to data and code is permitted.
    \end{itemize}

\item {\bf Experimental setting/details}
    \item[] Question: Does the paper specify all the training and test details (e.g., data splits, hyperparameters, how they were chosen, type of optimizer) necessary to understand the results?
    \item[] Answer: \answerYes{} 
    \item[] Justification: We provide information on the hyperparameter selection and datasets in Appendix \ref{app:detials_all_exp}.
    \item[] Guidelines:
    \begin{itemize}
        \item The answer \answerNA{} means that the paper does not include experiments.
        \item The experimental setting should be presented in the core of the paper to a level of detail that is necessary to appreciate the results and make sense of them.
        \item The full details can be provided either with the code, in appendix, or as supplemental material.
    \end{itemize}

\item {\bf Experiment statistical significance}
    \item[] Question: Does the paper report error bars suitably and correctly defined or other appropriate information about the statistical significance of the experiments?
    \item[] Answer: \answerYes{} 
    \item[] Justification: For image datasets, we do not report error bars, as the train/test splits are predefined. For tabular datasets, we report standard errors, which are provided in Appendix \ref{app_tabular_dhn}.
    \item[] Guidelines:
    \begin{itemize}
        \item The answer \answerNA{} means that the paper does not include experiments.
        \item The authors should answer \answerYes{} if the results are accompanied by error bars, confidence intervals, or statistical significance tests, at least for the experiments that support the main claims of the paper.
        \item The factors of variability that the error bars are capturing should be clearly stated (for example, train/test split, initialization, random drawing of some parameter, or overall run with given experimental conditions).
        \item The method for calculating the error bars should be explained (closed form formula, call to a library function, bootstrap, etc.)
        \item The assumptions made should be given (e.g., Normally distributed errors).
        \item It should be clear whether the error bar is the standard deviation or the standard error of the mean.
        \item It is OK to report 1-sigma error bars, but one should state it. The authors should preferably report a 2-sigma error bar than state that they have a 96\% CI, if the hypothesis of Normality of errors is not verified.
        \item For asymmetric distributions, the authors should be careful not to show in tables or figures symmetric error bars that would yield results that are out of range (e.g., negative error rates).
        \item If error bars are reported in tables or plots, the authors should explain in the text how they were calculated and reference the corresponding figures or tables in the text.
    \end{itemize}

\item {\bf Experiments compute resources}
    \item[] Question: For each experiment, does the paper provide sufficient information on the computer resources (type of compute workers, memory, time of execution) needed to reproduce the experiments?
    \item[] Answer: \answerYes{} 
    \item[] Justification: We provide information on the computer resources in Appendix \ref{app:detials_all_exp}.
    \item[] Guidelines:
    \begin{itemize}
        \item The answer \answerNA{} means that the paper does not include experiments.
        \item The paper should indicate the type of compute workers CPU or GPU, internal cluster, or cloud provider, including relevant memory and storage.
        \item The paper should provide the amount of compute required for each of the individual experimental runs as well as estimate the total compute. 
        \item The paper should disclose whether the full research project required more compute than the experiments reported in the paper (e.g., preliminary or failed experiments that didn't make it into the paper). 
    \end{itemize}
    
\item {\bf Code of ethics}
    \item[] Question: Does the research conducted in the paper conform, in every respect, with the NeurIPS Code of Ethics \url{https://neurips.cc/public/EthicsGuidelines}?
    \item[] Answer: \answerYes{} 
    \item[] Justification: This research conforms to the NeurIPS Code of Ethics.
    \item[] Guidelines:
    \begin{itemize}
        \item The answer \answerNA{} means that the authors have not reviewed the NeurIPS Code of Ethics.
        \item If the authors answer \answerNo, they should explain the special circumstances that require a deviation from the Code of Ethics.
        \item The authors should make sure to preserve anonymity (e.g., if there is a special consideration due to laws or regulations in their jurisdiction).
    \end{itemize}

\item {\bf Broader impacts}
    \item[] Question: Does the paper discuss both potential positive societal impacts and negative societal impacts of the work performed?
    \item[] Answer: \answerNA{} 
    \item[] Justification: This work does not have any direct societal impact.
    \item[] Guidelines:
    \begin{itemize}
        \item The answer \answerNA{} means that there is no societal impact of the work performed.
        \item If the authors answer \answerNA{} or \answerNo, they should explain why their work has no societal impact or why the paper does not address societal impact.
        \item Examples of negative societal impacts include potential malicious or unintended uses (e.g., disinformation, generating fake profiles, surveillance), fairness considerations (e.g., deployment of technologies that could make decisions that unfairly impact specific groups), privacy considerations, and security considerations.
        \item The conference expects that many papers will be foundational research and not tied to particular applications, let alone deployments. However, if there is a direct path to any negative applications, the authors should point it out. For example, it is legitimate to point out that an improvement in the quality of generative models could be used to generate Deepfakes for disinformation. On the other hand, it is not needed to point out that a generic algorithm for optimizing neural networks could enable people to train models that generate Deepfakes faster.
        \item The authors should consider possible harms that could arise when the technology is being used as intended and functioning correctly, harms that could arise when the technology is being used as intended but gives incorrect results, and harms following from (intentional or unintentional) misuse of the technology.
        \item If there are negative societal impacts, the authors could also discuss possible mitigation strategies (e.g., gated release of models, providing defenses in addition to attacks, mechanisms for monitoring misuse, mechanisms to monitor how a system learns from feedback over time, improving the efficiency and accessibility of ML).
    \end{itemize}
    
\item {\bf Safeguards}
    \item[] Question: Does the paper describe safeguards that have been put in place for responsible release of data or models that have a high risk for misuse (e.g., pre-trained language models, image generators, or scraped datasets)?
    \item[] Answer: \answerNA{} 
    \item[] Justification: The proposed method is not considered a high-risk model.
    \item[] Guidelines:
    \begin{itemize}
        \item The answer \answerNA{} means that the paper poses no such risks.
        \item Released models that have a high risk for misuse or dual-use should be released with necessary safeguards to allow for controlled use of the model, for example by requiring that users adhere to usage guidelines or restrictions to access the model or implementing safety filters. 
        \item Datasets that have been scraped from the Internet could pose safety risks. The authors should describe how they avoided releasing unsafe images.
        \item We recognize that providing effective safeguards is challenging, and many papers do not require this, but we encourage authors to take this into account and make a best faith effort.
    \end{itemize}

\item {\bf Licenses for existing assets}
    \item[] Question: Are the creators or original owners of assets (e.g., code, data, models), used in the paper, properly credited and are the license and terms of use explicitly mentioned and properly respected?
    \item[] Answer: \answerYes{} 
    \item[] Justification: We provide the implementation details and corresponding citations for the models in Appendix \ref{app:detials_all_exp}.
    \item[] Guidelines:
    \begin{itemize}
        \item The answer \answerNA{} means that the paper does not use existing assets.
        \item The authors should cite the original paper that produced the code package or dataset.
        \item The authors should state which version of the asset is used and, if possible, include a URL.
        \item The name of the license (e.g., CC-BY 4.0) should be included for each asset.
        \item For scraped data from a particular source (e.g., website), the copyright and terms of service of that source should be provided.
        \item If assets are released, the license, copyright information, and terms of use in the package should be provided. For popular datasets, \url{paperswithcode.com/datasets} has curated licenses for some datasets. Their licensing guide can help determine the license of a dataset.
        \item For existing datasets that are re-packaged, both the original license and the license of the derived asset (if it has changed) should be provided.
        \item If this information is not available online, the authors are encouraged to reach out to the asset's creators.
    \end{itemize}

\item {\bf New assets}
    \item[] Question: Are new assets introduced in the paper well documented and is the documentation provided alongside the assets?
    \item[] Answer: \answerYes{} 
    \item[] Justification: All implementation details are provided in the Appendix \ref{app:detials_all_exp}, and the code used in this work will be released upon acceptance of the paper.    \item[] Guidelines: 
    \begin{itemize}
        \item The answer \answerNA{} means that the paper does not release new assets.
        \item Researchers should communicate the details of the dataset\slash code\slash model as part of their submissions via structured templates. This includes details about training, license, limitations, etc. 
        \item The paper should discuss whether and how consent was obtained from people whose asset is used.
        \item At submission time, remember to anonymize your assets (if applicable). You can either create an anonymized URL or include an anonymized zip file.
    \end{itemize}

\item {\bf Crowdsourcing and research with human subjects}
    \item[] Question: For crowdsourcing experiments and research with human subjects, does the paper include the full text of instructions given to participants and screenshots, if applicable, as well as details about compensation (if any)? 
    \item[] Answer: \answerNA{} 
    \item[] Justification: We did not conduct any experiments involving human subjects.
    \item[] Guidelines:
    \begin{itemize}
        \item The answer \answerNA{} means that the paper does not involve crowdsourcing nor research with human subjects.
        \item Including this information in the supplemental material is fine, but if the main contribution of the paper involves human subjects, then as much detail as possible should be included in the main paper. 
        \item According to the NeurIPS Code of Ethics, workers involved in data collection, curation, or other labor should be paid at least the minimum wage in the country of the data collector. 
    \end{itemize}

\item {\bf Institutional review board (IRB) approvals or equivalent for research with human subjects}
    \item[] Question: Does the paper describe potential risks incurred by study participants, whether such risks were disclosed to the subjects, and whether Institutional Review Board (IRB) approvals (or an equivalent approval/review based on the requirements of your country or institution) were obtained?
    \item[] Answer: \answerNA{} 
    \item[] Justification: We did not conduct any experiments involving human subjects.
    \item[] Guidelines:
    \begin{itemize}
        \item The answer \answerNA{} means that the paper does not involve crowdsourcing nor research with human subjects.
        \item Depending on the country in which research is conducted, IRB approval (or equivalent) may be required for any human subjects research. If you obtained IRB approval, you should clearly state this in the paper. 
        \item We recognize that the procedures for this may vary significantly between institutions and locations, and we expect authors to adhere to the NeurIPS Code of Ethics and the guidelines for their institution. 
        \item For initial submissions, do not include any information that would break anonymity (if applicable), such as the institution conducting the review.
    \end{itemize}

\item {\bf Declaration of LLM usage}
    \item[] Question: Does the paper describe the usage of LLMs if it is an important, original, or non-standard component of the core methods in this research? Note that if the LLM is used only for writing, editing, or formatting purposes and does \emph{not} impact the core methodology, scientific rigor, or originality of the research, declaration is not required.
    \item[] Answer: \answerNA{} 
    \item[] Justification: We used LLMs only for writing.
    \item[] Guidelines:
    \begin{itemize}
        \item The answer \answerNA{} means that the core method development in this research does not involve LLMs as any important, original, or non-standard components.
        \item Please refer to our LLM policy in the NeurIPS handbook for what should or should not be described.
    \end{itemize}

\end{enumerate}

\end{document}